\documentclass[11pt]{article}

\usepackage[top=30truemm,bottom=30truemm,left=25truemm,right=25truemm]{geometry}
\usepackage[utf8]{inputenc} 
\usepackage[T1]{fontenc}    

\usepackage{microtype}
\usepackage{graphicx}
\usepackage{subcaption}
\usepackage{booktabs} 

\usepackage[colorlinks=true,citecolor=blue,linkcolor=blue]{hyperref}

\usepackage{amsmath,amssymb,amsthm}
\usepackage{mathtools}
\usepackage{bm,bbm}
\usepackage{physics}
\usepackage{xcolor}
\usepackage{dsfont}
\usepackage[ruled,linesnumbered]{algorithm2e}
\usepackage{tgtermes}

\usepackage[capitalize,noabbrev,nameinlink]{cleveref}

\usepackage[textsize=tiny]{todonotes}

\usepackage[authoryear,round]{natbib} 
\usepackage[font=small,labelfont=bf]{caption}

\theoremstyle{plain}
\newtheorem{theorem}{Theorem}

\newtheorem{lemma}[theorem]{Lemma}

\newtheorem{remark}{Remark}
\newtheorem{claim}{Claim}
\theoremstyle{definition}

\theoremstyle{remark}

\newcommand\numberthis{\addtocounter{equation}{1}\tag{\theequation}}
\newcommand{\argmin}{\mathop{\arg\min}}

\newcommand{\m}{\middle|}
\newcommand{\R}{\mathcal{R}}
\newcommand{\nec}{\mathcal{A}}
\newcommand{\B}{\mathcal{B}}

\newcommand{\sP}{\mathbb{P}}
\newcommand{\E}{\mathbb{E}}
\renewcommand{\rank}{\text{\textup{Rank}}}
\newcommand{\dis}{\mathcal{D}_\alpha}
\newcommand{\Fdis}{\mathcal{F}_\alpha}
\newcommand{\Pdis}{\mathcal{P}_\alpha}
\newcommand{\normv}{\left\lVert \bm{v}\right\rVert}
\newcommand{\ind}{\mathbbm{1}}
\newcommand{\pn}[1]{^{(#1)}}

\newcommand{\wm}{\widetilde{m}}

\title{Note on Follow-the-Perturbed-Leader in \\Combinatorial Semi-Bandit Problems}
\author{Botao Chen\footnote{
    Kyoto University; \texttt{chen.botao.63r@st.kyoto-u.ac.jp}.
} \and Junya Honda\footnote{
    Kyoto University and RIKEN AIP; \texttt{honda@i.kyoto-u.ac.jp}.
}}

\begin{document}
\maketitle

\begin{abstract}
    This paper studies the optimality and complexity of Follow-the-Perturbed-Leader (FTPL) policy in size-invariant combinatorial semi-bandit problems. Recently, \citet{pmlr-v201-honda23a} and \citet{pmlr-v247-lee24a} showed that FTPL achieves Best-of-Both-Worlds (BOBW) optimality in standard multi-armed bandit problems with Fr\'{e}chet-type distributions. However, the optimality of FTPL in combinatorial semi-bandit problems remains unclear. In this paper, we consider the regret bound of FTPL with geometric resampling (GR) in size-invariant semi-bandit setting, showing that FTPL respectively achieves $O\qty(\sqrt{m^2 d^\frac{1}{\alpha}T}+\sqrt{mdT})$ regret with Fr\'{e}chet distributions, and the best possible regret bound of $O\qty(\sqrt{mdT})$ with Pareto distributions in adversarial setting. Furthermore, we extend the conditional geometric resampling (CGR) to size-invariant semi-bandit setting, which reduces the computational complexity from $O(d^2)$ of original GR to $O\qty(md\qty(\log(d/m)+1))$ without sacrificing the regret performance of FTPL.
\end{abstract}

\section{Introduction}
The combinatorial semi-bandit is a sequential decision-making problem under uncertainty, which is a generalization of the classical multi-armed bandit problem. 
It is instrumental in many practical applications, such as recommender systems \citep{wang2017efficient}, online advertising \citep{nuara2022online}, crowdsourcing \citep{ul2016efficient}, adaptive routing \citep{gai2012combinatorial} and network optimization \citep{kveton2014matroid}. 
In this problem, the learner chooses an action $a_t$ from an action set $\nec\subset\qty{0,1}^d$, where $d\in\mathbb{N}$ is the dimension of the action set. 
In each round $t\in[T]=\qty{1,2,\cdots,T}$, the loss vector $\ell_t=\qty(\ell_{t,1},\ell_{t,2},\cdots,\ell_{t,d})$ is determined by the environment, and the learner incurs a loss $\left\langle \ell_t,a_t\right\rangle$ and can only observe the loss $\ell_{t,i}$ for all $i\in[d]$ such that $a_{t,i}=1$.
The goal of the learner is to minimize the cumulative loss over all the rounds.
The performance of the learner is often measured in terms of the \textit{pseudo-regret} defined as 
$\mathcal{R}(T)=\E\qty[\sum_{t=1}^{T}\left\langle \ell_t,a_t\right\rangle]-\min_{a\in\nec}\E\qty[\sum_{t=1}^{T}\left\langle \ell_t,a\right\rangle]$,
which describes the gap between the expected cumulative loss of the learner and of the optimal arm fixed in hindsight.

Since the introduction by \citet{chen2013combinatorial}, combinatorial semi-bandit problem has been widely studied, which mainly focus on two settings on the formulation of the environment to decide the loss vector, namely the stochastic setting and the adversarial setting.
In the stochastic setting, the sequence of loss vectors $\qty(\ell_t)_{t=1}^T$ is assumed to be independent and identically distributed (i.i.d.) from an unknown but fixed distribution $\mathcal{D}$ over $[0,1]^d$ with mean 
$\mu=\E_{\ell\sim\mathcal{D}}\qty[\ell]$.
The fixed single optimal action is defined as 
$a^*=\argmin_{a\in\nec}\E\qty[\sum_{t=1}^{T}\left\langle \ell_t,a\right\rangle]$, and we present the minimum suboptimality gap as 
$\Delta=\min_{a\in\nec\setminus\{a^*\}}\qty{\mu^{\top}(a-a^*)}$.
CombUCB \citep{kveton2015tight} and Combinatorial Thompson Sampling \citep{wang2018thompson} can achieves a gap-dependent regret bounds of $O\qty(dm\log T/\Delta)$ for general action sets and $O\qty((d-m)\log T/\Delta)$ for matroid semi-bandits, where $m=\max_{a\in\nec}\left\lVert a\right\rVert $ denotes the maximum size of any action in the set $\nec$.

In the adversarial setting, the loss vectors $\ell_t$ is determined from $[0,1]^d$ by an adversary in an arbitrary manner, which are not assumed to follow any specific distribution \citep{kveton2015tight,neu2015first,wang2018thompson}. 
For this setting, the regret bound of $O\qty(\sqrt{mdT})$ can be achieved by some policies, such as OSMD \citep{audibert2014regret} and FTRL with hybrid-regularizer \citep{zimmert2019beating}, which matches the lower bound of $\Omega\qty(\sqrt{mdT})$ \citep{audibert2014regret}.

In practical scenarios, the environment to determine the loss vectors is often unknown. Therefore, policies that can adaptively address both stochastic and adversarial settings have been widely studied, particularly in the context of standard multi-armed bandit problems. The Tsallis-INF policy \citep{zimmert2021tsallis} policy, which is based on Follow-the-Regularized-Leader (FTRL), has demonstrated the ability to achieve the optimality in both settings. For combinatorial semi-bandit problems, there also exist some work on this topic \citep{wei2018more,zimmert2019beating,ito2021hybrid,tsuchiya2023further}.

Howerver, some BOBW policies, such as FTRL, require an explicit computation of the arm-selection probability by solving a optimization problem. This leads to computational inefficiencies, and the complexity substantially increase in for combinatorial semi-bandits. In light of this limitation, the Follow-the-Perturbed-Leader (FTPL) policy, has gained significant attention due to its optimization-free nature. 
Recently, \citet{pmlr-v201-honda23a} and \citet{pmlr-v247-lee24a} demonstrated that FTPL achieves the Best-of-Both-Worlds (BOBW) optimality in standard multi-armed bandit problems with Fr\'{e}chet-type perturbations, which inspires researchers to explore the optimality of FTPL in combinatorial semi-bandit problems.
A preliminary effort by \citet{zhan2025follow} aimed to tackle this setting, though their analysis contains a technical flaw.
In fact, the analysis becomes substantially more complex in the combinatorial semi-bandit setting, which needs more furter investigation. 

\paragraph{Contributions of This Paper}

Firstly, we investigate the optimality of FTPL with geometric resampling with Fr\'{e}chet or Pareto distributions in adversarial size-invariant semi-bandit problems. We show that FTPL respectively achieves $O\qty(\sqrt{m^2 d^\frac{1}{\alpha}T}+\sqrt{mdT})$ regret with Fr\'{e}chet distributions, and the best possible regret bound of $O\qty(\sqrt{mdT})$ with Pareto distributions in this setting. To the best of our knowledge, this is the first work that provides a correct proof of the regret bound for FTPL with Fr\'{e}chet-type distributions in adversarial combinatorial semi-bandit problems.
Furthermore, we extend the technique called Conditional Geometric Resampling (CGR) \citep{chen2025cgr} to size-invariant semi-bandit setting, which reduces the computational complexity from $O(d^2)$ of the original GR to $O\qty(md\qty(\log(d/m)+1))$ without sacrificing the regret guarantee of the one with the original GR.

\subsection{Related Work}

\subsubsection{\texorpdfstring{Technical Issues in \citet{zhan2025follow}}{Technical Issues in Zhan et al. (2025)}}

The most closely related work is by \citet{zhan2025follow}. In their paper, they consider the FTPL policy with Fr\'{e}chet distribution with shape $2$ in the size-invariant semi-bandits, which is a special case of combinatorial semi-bandit problems. They provide a proof to claim that FTPL with Fr\'{e}chet distribution with shape $2$ achieves $O\qty(\sqrt{md\log(d)T})$ regret in adversarial setting and a logrithmic regret bound in stochastic setting. However, their proof includes a serious issue that renders the main result incorrect, which is explained below.

A function is analyzed in Lemma~4.1 in \citet{zhan2025follow}, which is later used in their evaluation of a component of the regret called the stability term. 
In this lemma, they evaluate the function for two cases, and the second case is just mentioned as ``can be shown by the same argument'' without a proof.
However, upon closer inspection, this claim cannot be justified by an analogy. In Section~\ref{sec:issue}, we highlight the detailed step where the analogy fails, and further support this observation with numerical verification, which demonstrates that the claimed result does not hold.

The main difficulty of the optimal regret analysis of FTPL lies in the analysis in the stability term \citep{pmlr-v201-honda23a,pmlr-v247-lee24a}, which is also the problem we mainly addressed. Unfortunately, this main difficulty lies behind these skipped or incorrect arguments, and thus we need an essentially new technique to complete the regret analysis for this problem.
In this paper, we evaluate the stability term in a totally different way in Lemmas~\ref{lem:lambda_star}--\ref{lem:stability}, which demonstrates that the stability term can be bounded by the maximum of simple quantities each of which is associated with a subset of base-arms.

\section{Problem Setup}

In this section, we formulate the problem and introduce the framework of FTPL with geometric resampling. 
We consider an action set $\nec\subset\qty{0,1}^d$, where each element $a\in\nec$ is called an action. 
For each base-arm $i\in[d]$, we assume that there exists at least an action $a\in\nec$ such that $a_i=1$.
In this paper, we consider a special case of action sets in combinatorial semi-bandit, referred to as size-invariant semi-bandit. 
In this setting, we define the action set $\nec=\qty{a:\left\lVert a\right\rVert_1=m}$, where $m$ is the number of selected base-arms at each round.
At each round $t\in[T]=\qty{1,2,\cdots,T}$, the environment determines a loss vector $\ell_t=\qty(\ell_{t,1},\ell_{t,2},\cdots,\ell_{t,d})^\top\in[0,1]^d$, and the learner takes an action $a_t\in\nec$ and incurs a loss $\left\langle \ell_t,a_t\right\rangle$. 
In the semi-bandit setting, the learner only observes the loss $\ell_{t,i}$ for all $i\in[d]$ such that $a_{t,i}=1$, whereas $\ell_{t,i}$ that corresponds to $a_{t,i}=0$ is not observed.

In this paper, we only consider the setting that the loss vector is determined in an adversarial way. 
In this setting, the loss vectors $\qty(\ell_t)_{t=1}^T$ are not assumed to follow any specific distribution, and they are determined in an arbitrary manner, which may depend on the past history of the actions and losses $\qty{\qty(\ell_s,a_s)}_{s=1}^{t-1}$. 

The performance of the learner is evaluated in terms of the pseudo-regret, which is defined as
\begin{equation*}
    \mathcal{R}(T)=\E\qty[\sum_{t=1}^{T}\left\langle \ell_t,a_t-a^*\right\rangle],\quad a^*\in\min_{a\in\nec}\E\qty[\sum_{t=1}^{T}\left\langle \ell_t,a\right\rangle].
\end{equation*}

\subsection{Follow-the-Perturbed-Leader}
\begin{table}[t]
    \centering
    \caption{Notation}
    \label{tab:notation}
    \begin{tabular}{ll}
    \toprule
    \textbf{Symbol} & \textbf{Meaning} \\
    \midrule
    $\nec \subset \{0,1 \}^d$ & Action set \\
    $d \in \mathbb{N}$ & Dimensionality of action set \\
    $m \leq d$ & $m = \|a\|_1$ for any $a\in\nec$ \\
    $\eta $ & Learning rate\\
    $\ell_t\in[0,1]^d$ & Loss vector\\
    $\hat{\ell}_t\in\qty[0,\infty]^d$ & Estimated loss vector\\
    $\hat{L}_t\in[0,\infty]^d$ & Cumulative estimated loss vector\\
    \midrule
    $\rank\qty(i,\bm{u};\B)$ & \parbox[t]{8cm}{Rank of $i$-th element of $\bm{u}$ in $\qty{u_j:j\in\B}$ in descending order, $\B$ omitted when $\B=[d]$} \\
    $\sigma_i$ & \parbox[t]{8cm}{The number of arms (including $i$ itself) whose cumulative losses do not exceed $\hat{L}_{t,i}$, i.e., $\rank(i,-\hat{L}_t)$}\\
    \midrule
    $\nu$ & Left-end point of perturbation\\
    $r_t\in[\nu,\infty]^d$ & $d$-dimensional perturbation\\
    \midrule
    $f(x)$ & Probability distribution function of perturbation \\
    $F(x)$ & Cumulative distribution function of perturbation \\
    \midrule
    $\mathcal{F}_\alpha$ & Fr\'{e}chet distribution with shape $\alpha$ \\
    $\mathcal{P}_\alpha$ & Pareto distribution with shape $\alpha$ \\
    \bottomrule
    \end{tabular}
\end{table}
\begin{algorithm}[t]
    \SetAlgoLined
    \caption{Follow-the-Perturbed-Leader}
    \label{alg:FTPL}
    \KwIn{Action set $\nec \subseteq \{0,1\}^d$, learning rate $\eta \in \mathbb{R}^+$\;}
    \textbf{Initialization: }$\hat{L}_1 \coloneqq\mathbf{0} \in \mathbb{R}^d$\;
    
    \For{$t=1, \dots, T$}{
        Sample $r_t=\qty(r_{t,1},r_{t,2},\cdots,r_{t,d})$ i.i.d. from $\mathcal{D}$\;
        Choose action $a_t = \argmin_{a \in \nec} \left\{a^\top (\eta \hat{L}_t - r_t)\right\}$ and observe $\qty{\ell_{t,i}:a_{t,i}=1}$\;
        Compute an estimator $\widehat{w_{t,i}^{-1}}$ for $w_{t,i}^{-1}$ by geometric resampling for all $i$ such that $a_{t,i}=1$\;
        Set $\hat{\ell}_t\coloneqq\sum_{i:a_{t,i}=1}\ell_{t,i}\widehat{w_{t,i}^{-1}}e_i$ and $\hat{L}_{t+1}\coloneqq\hat{L}_t+\hat{\ell}_t$\;
    }
\end{algorithm}

We consider the Follow-the-Perturbed-Leader (FTPL) policy, whose entire procedure is given in Algorithm~\ref{alg:FTPL}.
In combinatorial semi-bandit problems, FTPL policy maintains a cumulative estimated loss $\hat{L}_t$ and plays an action
\begin{equation*}
    a_t = \argmin_{a \in \nec} \left\{a^\top (\eta \hat{L}_t - r_t)\right\},
\end{equation*}
where $\eta\in\mathbb{R}^+$ is the learning rate, and $r_t=\qty(r_{t,1},r_{t,2},\cdots,r_{t,d})$ denotes the random perturbation i.i.d. from a common distribution $\mathcal{D}$ with a distribution function $F$.
In this paper, we consider two types of perturbation distributions.
The first is Fr\'{e}chet distribution $\Fdis$, with the probability density function $f(x)$ and the cumulative distribution function $F(x)$ given by
\begin{equation*}
    f(x)=\alpha x^{-(\alpha+1)}e^{-1/x^{\alpha}},\quad
    F(x)=e^{-1/x^{\alpha}},\quad x\geq 0, \alpha>1.
\end{equation*}
The second is Pareto distribution $\Pdis$, whose density and cumulative distribution functions are defined as
\begin{equation*}
    f(x)=\alpha x^{-(\alpha+1)},\quad
    F(x)=1-x^{-\alpha},\quad x\geq 1, \alpha>1.
\end{equation*}
Denote the rank of $i$-th element in $\bm{u}$ in descending order as $\rank\qty(i,\bm{u})$. The probability of selecting base-arm $i$ and $\rank\qty(i, r-\lambda ) = \theta\in[m]$ given $\hat{L}_t$ is written as $\phi_{i,\theta}(\eta\hat{L}_t;\mathcal{D})$. Then, for $\lambda\in[0,\infty)^d$, letting $\widetilde{\lambda}_1,\widetilde{\lambda}_2,\cdots,\widetilde{\lambda}_d$ be the sorted elements of $\lambda$ such that $\widetilde{\lambda}_1\leq\widetilde{\lambda}_2\leq\cdots\leq\widetilde{\lambda}_d$, $\phi_{i,\theta}(\lambda;\mathcal{D})$ can be expressed as
\begin{align*}
    \phi_{i,\theta}(\lambda;\mathcal{D})
    &\coloneqq
    \sP_{r=\qty(r_1,\dots,r_d) \sim \mathcal{D}} \qty[ \rank\qty(i, r-\lambda ) = \theta]\\
    &=\int_{\nu-\widetilde{\lambda}_\theta}^\infty \sum_{\bm{v}\in\mathcal{S}_{i,\theta}}\qty(\prod_{j:v_j=1}\qty(1-F(z+\lambda_j))\prod_{j:v_j=0,j\neq i}F(z+\lambda_j)) \dd F(z+\lambda_i),\numberthis\label{eq:arm-rank_probability}
\end{align*}
where $\mathcal{S}_{i,\theta}$ is defined as
$\mathcal{S}_{i,\theta}=\qty{\bm{v}\in\qty{0,1}^d:\normv_1=\theta-1,v_i=0}$.
Here, $\nu$ denotes the left endpoint of the support of $F$.

Then, we write the probability of selecting the base-arm $i$ as $w_{t,i}=\phi_i(\eta\hat{L}_t;\mathcal{D})$, where for $\lambda\in[0,\infty)^d$
\begin{align*}
    \phi_i(\lambda;\mathcal{D}) &\coloneqq
    \sum_{\theta=1}^{m}\mathbb{P}_{r=\qty(r_1,\dots,r_d) \sim \mathcal{D}} \qty[ \rank\qty(i, r-\lambda ) = \theta]
    =\sum_{\theta=1}^{m}\phi_{i,\theta}(\lambda;\mathcal{D})\\
    &=
    \sum_{\theta=1}^{m}\int_{\nu-\widetilde{\lambda}_\theta}^\infty \sum_{\bm{v}\in\mathcal{S}_{i,\theta}}\qty(\prod_{j:v_j=1}\qty(1-F(z+\lambda_j))\prod_{j:v_j=0,j\neq i}F(z+\lambda_j)) \dd F(z+\lambda_i).\numberthis\label{eq:arm_probability}
\end{align*}

Table~\ref{tab:notation} summarizes the notation used in this paper.

\subsection{Geometric Resampling}

\begin{algorithm}[t]
    \caption{Geometric Resampling}
    \label{alg:GR}
    \KwIn{Chosen action $a_t$, action set $\nec$, cumulative loss $\hat{L}_t$, learning rate $\eta$}
    Set $K \coloneqq \bm{0}\in\mathbb{R}^d$; $s\coloneqq a_t$\;
    \Repeat{$s=\bm{0}$}{
        $K \coloneqq K + s$\;
        Sample $r_t^{\prime} = (r_{t,1}^{\prime}, r_{t,2}^{\prime}, \dots, r_{t,d}^{\prime})$ i.i.d. from $\mathcal{D}$\;
        $a'_t = \argmin_{a \in \nec} \left\{a^\top (\eta \hat{L}_t - r_t)\right\}$\;
        $s\coloneqq s\circ\qty(\bm{1}_d-a'_t)$\tcp*{$\bm{1}_d$ denotes the $d$-dimensional all-ones vector}
    }
    Set $\widehat{w_{t,i}^{-1}} \coloneqq K_i$ for all $i$ such that $a_{t,i}=1$\;
\end{algorithm}

Since the loss in every round is partially observable in the setting of semi-bandit feedback, for the unbiased loss estimator, many policies generally use an estimator $\hat{\ell}_t$ for the loss vector $\ell_t$. 
Then, the cumulative estimated loss vector $\hat{L}_t$ is obtained as $\hat{L}_t=\sum_{s=1}^{t-1}\hat{\ell}_s$. 
In standard multi-armed bandit problems, many policies like FTRL often employ an importance-weighted (IW) estimator 
$\hat{\ell}_t=\qty(\ell_{t,I_t}/w_{t,I_t})e_{t,I_t}$, where $I_t$ means the chosen arm at round $t$, and the arm-selection probability $w_{t,I_t}$ is explicitly computed. 
However, in the combinatorial semi-bandit setting, the individual probability of each base-arm $i$ is not available, which complicates the construction of the unbiased estimator. 
To address this issue, \citet{JMLR:v17:15-091} proposed a technique called Geometric Resampling (GR). 
With this technique, for each selected base-arm $i$, FTPL policy can efficiently compute an unbiased estimator $\widehat{w_{t,i}^{-1}}$ for $w_{t,i}^{-1}$. 
The procedure of GR is shown in Algorithm~\ref{alg:GR}, where the notation $a\circ b$ denotes the element-wise product of two vectors $a$ and $b$, i.e., $\qty(a\circ b)_i = a_i b_i$ for all $i$.

Now we consider the computational complexity of GR. Let $M_{t,i}$ denote the number of resampling taken by geometric resampling at round $t$ until $s_i$ switches from $1$ to $0$. Here, $M_{t,i}$ is equal to $\widehat{w_{t,i}^{-1}}$ in GR. Then, the expected number of total resampling $M_t=\max_{i:a_{t,i}=1}M_{t,i}$ given $\hat{L}_t$ can be bounded as
\begin{align*}
    \mathbb{E}\qty[M_t\m\hat{L}_t]&=\mathbb{E}\qty[\max_{i:a_{t,i}=1}M_{t,i}\m\hat{L}_t]\\
    &\leq\mathbb{E}\qty[\sum_{i=1}^d a_{t,i}M_{t,i}\m\hat{L}_t]\\
    &=\sum_{i=1}^d\mathbb{E}\qty[a_{t,i}\m\hat{L}_t]\mathbb{E}\qty[M_{t,i}\m\hat{L}_t,a_{t,i}]\\
    &=\sum_{i=1}^d w_{t,i}\cdot\frac{1}{w_{t,i}} =d.
\end{align*}

For each resampling, the generation of perturbation requires the complexity of $O(d)$, 
and thus the total complexity of GR at each round is $O(d^2)$, which is independent of $w_t$. 
Compared with many other policies, in combinatorial semi-bandit problems, FTPL with GR is computationally more efficient. 
However, in standard $K$-armed bandit problems, the computational complexity of FTPL with GR still maintains $O(K^2)$. 
Though FTPL remains efficient for moderate size of $K$ \citep{pmlr-v201-honda23a}, primarily thanks to its optimization-free nature, the running time increases substantially as $K$ grows to a large number. 
To overcome this limitation, \citet{chen2025cgr} proposed an improved technique called Conditional Geometric Resampling (CGR), which reduces the complexity to $O(K\log K)$ and demonstrates superior runtime performance. 
Inspired by this, we extend CGR to size-invariant semi-bandit setting in this paper, which is presented in Section~\ref{sec:algorithm}.

\section{Regret Bounds}

In this section, we summarize the regret bound of FTPL in adversarial size-invariant semi-bandit setting.

\subsection{Main Results}
Combining Lemma~\ref{lem:regret_decomposition} from this section and Lemma~\ref{lem:stability} given in Section~\ref{sec:stability}, we can obtain the regret bound of FTPL in adversarial setting in the following theorem.

\begin{theorem}
    In the adversarial setting, when perturbation follows Fr\'{e}chet distribution with shape $\alpha>1$, FTPL with learning rate 
    \begin{equation*}
        \eta=\sqrt{\frac{\qty(\frac{\alpha}{\alpha-1}m^{1-\frac{1}{\alpha}}+\Gamma\qty(1-\frac{1}{\alpha}))\qty(d+1)^{\frac{1}{\alpha}}+m}{2(\alpha+1)\qty(m+\frac{1}{\alpha})^{\frac{1}{\alpha}}\qty(m+\frac{\alpha}{\alpha-1}(d-m+1)^{1-1/\alpha})T}}
    \end{equation*}
    satisfies
    \begin{multline*}
        \R(T)\leq \\
        2\qty(2(\alpha+1)\qty(m+\frac{1}{\alpha})^{\frac{1}{\alpha}}\qty(m+\frac{\alpha}{\alpha-1}(d-m+1)^{1-1/\alpha})\qty(\qty(\frac{\alpha}{\alpha-1}m^{1-\frac{1}{\alpha}}+\Gamma\qty(1-\frac{1}{\alpha}))\qty(d+1)^{\frac{1}{\alpha}}+m)T)^\frac{1}{2},
    \end{multline*}
    whose order is $O\qty(\sqrt{m^2 d^\frac{1}{\alpha}T}+\sqrt{mdT})$, where $\Gamma(\cdot)$ is the gamma function.
    When perturbation follows Pareto distribution with shape $\alpha>1$, FTPL with learning rate 
    \begin{equation*}
        \eta=\sqrt{\frac{\qty(\alpha m^{1-\frac{1}{\alpha}}+\qty(\alpha-1)\Gamma\qty(1-\frac{1}{\alpha}))\qty(d+1)^{\frac{1}{\alpha}}}{4\alpha^2\qty(m+\frac{1}{\alpha})^{\frac{1}{\alpha}}d^{1-1/\alpha}T}}
    \end{equation*}
    satisfies
    \begin{equation*}
        \R(T)\leq \frac{4\alpha}{\qty(\alpha-1)^\frac{1}{2}}\qty(\qty(m+\frac{1}{\alpha})^{\frac{1}{\alpha}}\qty(\frac{\alpha}{\alpha-1}m^{1-\frac{1}{\alpha}}+\Gamma\qty(1-\frac{1}{\alpha}))d^{1-1/\alpha}\qty(d+1)^{\frac{1}{\alpha}}T)^{\frac{1}{2}},
    \end{equation*}
    whose order is $O\qty(\sqrt{mdT})$.
\end{theorem}

\subsection{Regret Decomposition}

To evaluate the regret of FTPL, we firstly decompose the regret which is expressed as
\begin{equation*}
    \mathcal{R}(T)=\E\qty[\sum_{t=1}^{T}\left\langle \ell_t,a_t-a^*\right\rangle]=\sum_{t=1}^{T}\E\qty[\left\langle \ell_t,a_t-a^*\right\rangle]=\sum_{t=1}^{T}\E\qty[\left\langle \hat{\ell}_t,a_t-a^*\right\rangle].
\end{equation*}
This can be decomposed in the following way, whose proof is given in Section~\ref{sec:proof_regret_decomposition}.
\begin{lemma}\label{lem:regret_decomposition}
    For any $\alpha>1$ and $\dis\in\qty{\Fdis,\Pdis}$,
    \begin{equation}\label{eq:regret_decomposition}
        \mathcal{R}(T) \leq 
        \begin{cases}
            \sum_{t=1}^{T} \mathbb{E} \left[ \left\langle \hat{\ell}_t, w_t - w_{t+1} \right\rangle \right]  
        + \frac{\qty(\frac{\alpha}{\alpha-1}m^{1-\frac{1}{\alpha}}+\Gamma\qty(1-\frac{1}{\alpha}))\qty(d+1)^{\frac{1}{\alpha}}+m}{\eta} & \text{if } \dis = \Fdis, \\
            \sum_{t=1}^{T} \mathbb{E} \left[ \left\langle \hat{\ell}_t, w_t - w_{t+1} \right\rangle \right]  
        + \frac{\qty(\frac{\alpha}{\alpha-1}m^{1-\frac{1}{\alpha}}+\Gamma\qty(1-\frac{1}{\alpha}))\qty(d+1)^{\frac{1}{\alpha}}}{\eta} & \text{if } \dis = \Pdis.
        \end{cases}
    \end{equation}
\end{lemma}
We refer to the first and second terms of \eqref{eq:regret_decomposition} as stability term and penalty term, respectively.
\section{Stability of Arm-selection Probability}\label{sec:stability}

In the standard multi-armed bandit problem, the core and most challeging part in analyzing the regret of FTPL lies on the analysis of the arm-selection probability function \citep{abernethy2015fighting,pmlr-v201-honda23a,pmlr-v247-lee24a}. 
This challenge is further amplified in the combinatorial semi-bandit setting, where the base-arm selection probability $\phi_i(\lambda;\mathcal{D})$ given in \eqref{eq:arm_probability} exhibits significantly greater complexity. 
To this end, this section firstly introduces some tools used in the analysis and then derives properties of $\phi_i(\lambda;\mathcal{D})$, which is the main difficulty of the analysis of FTPL.

\subsection{General Tools for Analysis}\label{subsec:tool}
Since the probability of some events in different base-arm set $\B\subset[d]$ will be considered in the subsequent analysis, we introduce the parameter $\B$ into $\phi_{i,\theta}(\cdot)$. 
Denote the rank of $i$-th element of $\bm{u}$ in $\qty{u_j:j\in\B}$ in descending order as $\rank\qty(i, \bm{u};\B)$.
We define
\begin{align*}
    \phi_{i,\theta}(\lambda;\dis,\B)
    &=
    \sP_{r=\qty(r_1,\dots,r_d) \sim \mathcal{D}} \qty[ \rank\qty(i, r-\lambda;\B) = \theta]\\
    &=\int_{\nu-\widetilde{\lambda}_\theta}^\infty \sum_{\bm{v}\in\mathcal{S}_{i,\theta}^\B}\qty(\prod_{j:v_j=1}\qty(1-F(z+\lambda_j))\prod_{j:v_j=0,j\in\B\setminus\qty{i}}F(z+\lambda_j)) \dd F(z+\lambda_i),\numberthis\label{eq:prob_theta_B}
\end{align*}
where $\mathcal{S}_{i,\theta}^\B=\qty{\bm{v}\in\qty{0,1}^d:\normv_1=\theta-1,v_i=0,\text{ and }v_j=0\text{ for all }j\notin\B}$.

Under this definition, $\phi_{i,\theta}(\lambda;\dis,\B)$ means the probability that the base-arm $i$ ranks $\theta$-th among the base-arm set $\B$.
Based on \eqref{eq:prob_theta_B}, we define 
\begin{equation*}
    \phi_i(\lambda;\dis,\wm,\B)=
    \sum_{\theta=1}^{\wm} \sP_{r=\qty(r_1,\dots,r_d) \sim \mathcal{D}} \qty[ \rank\qty(i, r-\lambda;\B) = \theta]=\sum_{\theta=1}^{\wm} \phi_{i,\theta}(\lambda;\dis,\B),
\end{equation*}
which represents the probability of selecting base-arm $i$ when the base-arm set and the number of selected base-arms are respectively set as $\B$ and $\wm$ in size-invariant semi-bandit setting.
Since the definition above is an extension of \eqref{eq:arm-rank_probability} and \eqref{eq:arm_probability}, we have 
\begin{equation*}
    \phi_{i,\theta}(\lambda;\dis)=\phi_{i,\theta}(\lambda;\dis,[d]),\quad\text{and}\quad\phi_i(\lambda;\dis)=\phi_i(\lambda;\dis,m,[d]).
\end{equation*}
For analysis on the derivative, based on \eqref{eq:prob_theta_B} we define
\begin{equation*}
    J_{i,\theta}(\lambda;\dis,\B)=
    \int_{\nu-\widetilde{\lambda}_\theta}^\infty \frac{1}{z+\lambda_i}\sum_{\bm{v}\in\mathcal{S}_{i,\theta}^\B}\qty(\prod_{j:v_j=1}\qty(1-F(z+\lambda_j))\prod_{j:v_j=0,j\in\B\setminus\qty{i}}F(z+\lambda_j)) \dd F(z+\lambda_i)
\end{equation*}
and
\begin{equation*}
    J_i(\lambda;\dis,\wm,\B)=\sum_{\theta=1}^{\wm} J_{i,\theta}(\lambda;\dis,\B).
\end{equation*}
When $\wm=m$ and $\B=[d]$, we simply write
\begin{equation*}
    J_{i,\theta}(\lambda;\dis)=J_{i,\theta}(\lambda;\dis,m,[d]),\quad\text{and}\quad J_i(\lambda;\dis)=J_i(\lambda;\dis,m,[d]).
\end{equation*}

In the following, we write $\sigma_i$ to denote the number of arms (including $i$ itself) whose cumulative losses do
not exceed $\hat{L}_{t,i}$, i.e., $\rank(i,\hat{L}_t)=\sigma_i$.
Without loss of generality, 
in the subsequent analysis we always assume $\lambda_1\leq\lambda_2\leq\cdots\leq\lambda_d$ (ties are broken arbitrarily) so that $\sigma_i=i$ for notational simplicity.
To derive an upper bound, we employ the tools introduced above to provide lemmas related to the relation between the base-arm selection probability and its derivatives.

\subsection{Important Lemmas}

\begin{lemma}\label{lem:lambda_star}
    It holds that
    \begin{equation*}
        \frac{J_i(\lambda;\dis)}{\phi_i(\lambda;\dis)}\leq
        \max_{\substack{w\in\qty{0}\cup[(m\land i)-1]\\\theta\in[(m\land i)-w]}}\qty{\frac{J_{i,\theta}(\lambda^*;\dis,\B_{i,w})}{\phi_{i,\theta}(\lambda^*;\dis,\B_{i,w})}},
    \end{equation*}
    where
    \begin{equation*}
        \B_{i,w}=
        \begin{cases}
            [i], & \text{if } w=0,\\
            [i]\setminus[w], & \text{if } w\in[i],
        \end{cases}
        \text{ and }
        \lambda_k^* = 
        \begin{cases}
            \lambda_i, & \text{if } k\leq i,\\
            \lambda_k, & \text{if } k> i.
        \end{cases}
    \end{equation*}
\end{lemma}

Based on this result, the following lemma holds. 

\begin{lemma}\label{lem:sigma_i}
    If $\rank(i,-\lambda)=\sigma_i$, that is, $\lambda_i$ is the $\sigma_i$-th smallest among $\lambda_1,\cdots,\lambda_d$ (ties are broken arbitrarily), then
    \begin{equation*}
        \frac{J_i(\lambda;\Fdis)}{\phi_i(\lambda;\Fdis)}\leq \qty(\frac{(\sigma_i\land m)+\frac{1}{\alpha}}{(\sigma_i-m+1)\lor 1})^{\frac{1}{\alpha}}
        \quad\text{and}\quad
        \frac{J_i(\lambda;\Pdis)}{\phi_i(\lambda;\Pdis)}\leq
        \frac{2\alpha}{\alpha+1}\qty(\frac{(\sigma_i\land m)+\frac{1}{\alpha}}{\sigma_i})^{\frac{1}{\alpha}}.
    \end{equation*}
\end{lemma}

Next, following the steps in \citet{pmlr-v201-honda23a} and \citet{pmlr-v247-lee24a}, we extend the analysis to the combinatorial semi-bandit setting and obtain the following lemma.

\begin{lemma}\label{lem:each_term}
    For any $i\in[d],\alpha>1,\eta\hat{L}_t$ and $\dis\in\qty{\Fdis,\Pdis}$, it holds that
    \begin{equation*}
        \E\qty[\hat{\ell}_{t,i}\qty(\phi_i\qty(\eta \hat{L}_t;\dis)-\phi_i\qty(\eta \qty(\hat{L}_t+\hat{\ell}_t);\dis))\m \hat{L}_t]\leq
        \begin{cases}
            2(\alpha+1)\eta\qty(\frac{(\sigma_i\land m)+\frac{1}{\alpha}}{(\sigma_i-m+1)\lor 1})^{\frac{1}{\alpha}}, & \text{if } \dis=\Fdis,\\
            4\alpha\eta\qty(\frac{(\sigma_i\land m)+\frac{1}{\alpha}}{\sigma_i})^{\frac{1}{\alpha}}, & \text{if } \dis=\Pdis.
        \end{cases}
    \end{equation*}
\end{lemma}
By using the above lemma, we can express the stability term as follows.
\begin{lemma}\label{lem:stability}
    For any $\eta\hat{L}_t$, $\alpha>1$ and $\dis\in\qty{\Fdis,\Pdis}$, it holds that
    \begin{multline*}
        \E\qty[\hat{\ell}_t \qty(\phi\qty(\eta \hat{L}_t;\dis)-\phi\qty(\eta \qty(\hat{L}_t+\hat{\ell}_t);\dis))\m \hat{L}_t]\leq\\
        \begin{cases}
            2(\alpha+1)\eta\qty(m+\frac{1}{\alpha})^{\frac{1}{\alpha}}\qty(m+\frac{\alpha}{\alpha-1}(d-m+1)^{1-1/\alpha}), & \text{if } \dis=\Fdis,\\
            \frac{4\alpha^2}{\alpha-1}\eta\qty(m+\frac{1}{\alpha})^{\frac{1}{\alpha}}d^{1-1/\alpha}, & \text{if } \dis=\Pdis.
        \end{cases}
    \end{multline*}
\end{lemma}

\section{Conditional Geometric Resampling for Size-Invariant Semi-Bandit}\label{sec:algorithm}

Building on the idea proposed by \citet{chen2025cgr}, this section introduces an extension of Conditional Geometric Resampling (CGR) to the size-invariant semi-bandit setting. 
This algorithm is designed to provide multiple unbiased estimators $\qty{w_{t,i}^{-1}:a_{t,i}=1}$ in a more efficient way, which is based on the following lemma.
\begin{lemma}\label{lem:general_idea}
    Let $\mathcal{E}_{t,i}$ be an be an arbitrary necessary condition for
    \begin{equation}\label{eq:termination}
        \qty[\argmin_{a \in \nec} \left\{a^\top (\eta \hat{L}_t - r''_t)\right\}]_i=1.
    \end{equation}
    Consider resampling of $r_t''$ from $\mathcal{D}$ conditioned on
    $\mathcal{E}_{t,i}$ until \eqref{eq:termination} is satisfied.
    Then, the number $M_{t,i}$ of resampling for base-arm $i$ satisfies
    \begin{equation*}
        \E[M_{t,i}|\hat{L}_t, a_{t,i}]=\frac{\sP[\mathcal{E}_{t,i}|\hat{L}_t, a_{t,i}]}{w_{t,i}}.
    \end{equation*}
\end{lemma}

From this lemma, we can use 
\begin{equation*}
    \widehat{{w}_{t,i}^{-1}}=\frac{M_{t,i}}{\sP[\mathcal{E}_{t,i}|\hat{L}_t, a_{t,i}]} 
\end{equation*}
as an unbiased estimator of $w_{t,i}^{-1}$ for $r_t''$ sampled from $\mathcal{D}$ conditioned on $\mathcal{E}_{t,i}$.

\begin{proof}
    Define
    \begin{equation*}
        \chi_{t,i}(r_t'') =
        \begin{cases} 
            1, & \text{if } \qty[\argmin_{a \in \nec} \left\{a^\top (\eta \hat{L}_t - r''_t)\right\}]_i=1, \\
            0, & \text{otherwise}.
        \end{cases}
    \end{equation*}
    Consider $w_{t,i}$, the probability that base-arm $i$ is selected, with the condition $\mathcal{E}_{t,i}$. $w_{t,i}$ can be expressed as 
    \begin{align*}
        w_{t,i} &= \sP[\chi_{t,i}(r_t'')=1|\hat{L}_t, a_{t,i}] \\
        &= \sP[\chi_{t,i}(r_t'')=1|\mathcal{E}_{t,i},\hat{L}_t, a_{t,i}]\sP[\mathcal{E}_{t,i}|\hat{L}_t, a_{t,i}]+\sP[\chi_{t,i}(r_t'')=1|\mathcal{E}_{t,i}^c,\hat{L}_t, a_{t,i}]\sP[\mathcal{E}_{t,i}^c|\hat{L}_t, a_{t,i}]. \numberthis{\label{eq: conditioned on nec_t or not}}
    \end{align*}
    Note that $\mathcal{E}_{t,i}$ is an arbitrary necessary condition for $\chi_{t,i}(r_t'')=1$, which implies that
    \begin{equation*}
        \sP[\chi_{t,i}(r_t'')=1|\mathcal{E}_{t,i}^c,\hat{L}_t, a_{t,i}] = 0.
    \end{equation*}
    Therefore, from \eqref{eq: conditioned on nec_t or not} we immediately obtain
    \begin{equation}
        w_{t,i} = \sP[\chi_{t,i}(r_t'')=1|\mathcal{E}_{t,i},\hat{L}_t, a_{t,i}]\sP[\mathcal{E}_{t,i}|\hat{L}_t, a_{t,i}].
        \label{eq: conditioned on nec_t}
    \end{equation}
    Now we consider the expected number of resampling $M_{t,i}$ for base-arm $i$. Recall that $r_t''$ is sampled from $\mathcal{D}$ conditioned on $\mathcal{E}_{t,i}$ until \eqref{eq:termination} is satisfied, that is, $\chi_{t,i}(r_t'')=1$. Then $M_{t,i}$ follows geometric distribution with probability mass function
    \begin{equation*}
        \sP[M_{t,i}=m|\hat{L}_t, a_{t,i}]=\qty(1-\sP[\chi_{t,i}(r_t'')=1|\mathcal{E}_{t,i},\hat{L}_t, a_{t,i}])^{m-1}\sP[\chi_{t,i}(r_t'')=1|\mathcal{E}_{t,i},\hat{L}_t, a_{t,i}].
    \end{equation*}
    Therefore, the expected number of resampling given $\hat{L}_t$ and $a_{t,i}$ is expressed as 
    \begin{align*}
        \E_{r_t'' \sim \mathcal{D}|\mathcal{E}_{t,i}}[M_{t,i}|\hat{L}_t, a_{t,i}] 
        &= \sP[\chi_{t,i}(r_t'')=1|\mathcal{E}_{t,i},\hat{L}_t, a_{t,i}]\sum_{n=1}^{\infty}n\qty(1-\sP[\chi_{t,i}(r_t'')=1|\mathcal{E}_{t,i},\hat{L}_t, a_{t,i}])^{n-1} \\
        &= \sP[\chi_{t,i}(r_t'')=1|\mathcal{E}_{t,i},\hat{L}_t, a_{t,i}] /\qty(\sP[\chi_{t,i}(r_t'')=1|\mathcal{E}_{t,i},\hat{L}_t, a_{t,i}])^2 \\
        &= 1/\sP[\chi_{t,i}(r_t'')=1|\mathcal{E}_{t,i},\hat{L}_t, a_{t,i}]. \numberthis{\label{eq: expected number}}
    \end{align*}
    Combining \eqref{eq: conditioned on nec_t} and \eqref{eq: expected number}, we obtain
    \begin{align*}
        \E_{r_t'' \sim \mathcal{D}|\mathcal{E}_{t,i}}[M_{t,i}|\hat{L}_t, a_{t,i}]=\frac{\sP[\mathcal{E}_{t,i}|\hat{L}_t, a_{t,i}]}{w_{t,i}}.
    \end{align*}
\end{proof}

For base arm $i$ such that $\sigma_i>m$, we now consider resampling $r''_t$ from the perturbation distribution $\mathcal{D}$ conditioned on 
$$\mathcal{E}_{t,i}=\qty{\left|\qty{ j:r''_{t,j}\leq r''_{t,i},\sigma_j\leq\sigma_i} \right|\leq m},$$
that is, the event that $r''_{t,i}$ lies among the top-$m$ largest of the base-arms $j$ whose cumulative estimated losses are no worse than $i$. 
By the symmetry nature of the i.i.d. perturbations, we can sample $r''_t$ from this conditional distribution with simple operation, which corresponds to Lines~\ref{line:theta}--\ref{line:swap_end} in Algorithm~\ref{alg:CGR}.
For each base-arm $i$ such that $\sigma_i\leq m$, the resampling procedure in our proposed algorithm is the same as the original GR. 
By Lemma~\ref{lem:general_idea}, we can derive the properties of CGR in size-invariant semi-bandit setting as follows.

\begin{algorithm}[t]
    \SetAlgoLined
    \caption{Conditional Geometric Resampling}
    \label{alg:CGR}
    \KwIn{Chosen action $a_t$, action set $\nec$, cumulative loss $\hat{L}_t$, learning rate $\eta$\;}
    Set $K \coloneqq \bm{0}\in\mathbb{R}^d$; $s\coloneqq a_t$; $U\coloneqq\varnothing $; $C\coloneqq\mathbf{1}_d\in\mathbb{R}^d$\;
    \For{$i = 1, \dots, d$\label{line:scan_begin}}{
        \If{$a_{t,i}=1\text{ and }\sigma_i>m$}{
                $U \coloneqq U \cup \{i\}$; $C_i\coloneqq\sigma_i/m$\;
        }
    }\label{line:scan_end}
        
    \Repeat{$s = 0$\label{line:perturbation_end}}{\label{line:perturbation_begin}
        $K \coloneqq K + s$\label{line:global_begin}\;

        Sample $r_t^{\prime} = (r_{t,1}^{\prime}, r_{t,2}^{\prime}, \dots, r_{t,d}^{\prime})$ i.i.d. from $\mathcal{D}$\;

        $a'_t \coloneqq \argmin_{a \in \nec} \left\{a^\top (\eta \hat{L}_t - r'_t)\right\}$\;

        Sample $\theta$ from $[m]$ uniformly at random\label{line:global_end}\label{line:theta}\;

        \For{$i\in U$}{
            Find $i'$ such that $r'_{t,i'}$ is the $\theta\text{-th}$ largest in $\qty{r'_{t,j}:\sigma_j\leq\sigma_i}$\label{line:swap_begin}\;
            Set $r''\coloneqq r'$\;
            Swap $r''_{t,i'}$ and $r''_{t,i}$\label{swap_end}\label{line:swap_end}\;
            $a'_{t,i}\coloneqq\qty[\argmin_{a \in \nec} \left\{a^\top (\eta \hat{L}_t - r''_t)\right\}]_i$\;
            \lIf{$a'_{t,i}=1$}{$U\coloneqq U\setminus\{i\}$}
        }
        $s \coloneqq s \circ \qty(\mathbf{1}_d-a'_t)$\label{line:global_2}\;
        }

    Set $\widehat{w_{t,i}^{-1}} \coloneqq C_i K_i$ for all $i$ such that $a_{t,i}=1$\;
\end{algorithm}

\begin{lemma}\label{lem:CGR}
    The sample $r''_t$ obtained by Algorithm~\ref{alg:CGR} for each base-arm $i$ such that $\sigma_i>m$ follows the conditional distribution of $\mathcal{D}$ given
    $$\mathcal{E}_{t,i}=\qty{\left|\qty{ j:r''_{t,j}\leq r''_{t,i},\sigma_j\leq\sigma_i} \right|\leq m}.$$
    In addition, for any $i\in[d]$,
    $$\widehat{w_{t,i}^{-1}}=\qty(\frac{\sigma_i}{m}\lor 1)M_{t,i}$$
    given Algorithm~\ref{alg:CGR} serves as an unbiased estimator for $w_{t,i}^{-1}$, and the number of resampling $M_t$ satisfies
    \begin{equation}\label{eq:resampling_bound}
        \E_{r'_t\sim \mathcal{D},r''_t\sim\mathcal{D}|\mathcal{E}_{t,i}}\qty[M_t\m \hat{L}_{t}]\leq m+m\log\qty(d/m).
    \end{equation}
\end{lemma}

\begin{proof}
    Let $\sP^*[\cdot]$ denote the probability distribution of $r''_t$ after the value-swapping operation, and $\rank_{i,j}$ denote the rank of $r''_{t,j}$ among $\qty{r''_{t,k}:\sigma_k\leq\sigma_i}$. Then, we have $\mathcal{E}_{t,i}=\qty{\rank_{i,i}\in[m]}$. Given $\hat{L}_t,a_{t,i}$ and $\theta$, for any realization $\theta_0$ in $[m]$ of $\theta$ we have
    \begin{align*}
        &\sP^*\qty[\bigcap\nolimits_{j:\sigma_j\leq\sigma_i} \left\{r''_{t,j} \leq x_j\right\}\m\hat{L}_t, a_{t,i},\theta=\theta_0] \\ 
        =& 
        \sum_{j:\sigma_j\leq\sigma_i}\sP\qty[\bigcap\nolimits_{k:\sigma_k\leq\sigma_i,i\notin\qty{j,i}} \left\{r''_{t,k} \leq x_k\right\},\, r''_{t,j} \leq x_i,\, r''_{t,i} \leq x_j,\, \rank_{i,j}=\theta_0 \m\hat{L}_t, a_{t,i}] \\
        =& 
        \sum_{j:\sigma_j\leq\sigma_i}\sP\qty[\bigcap\nolimits_{i: \sigma_i\leq\sigma_i,i\notin\qty{j,i}} \left\{r''_{t,i} \leq x_i\right\},\, r''_{t,j} \leq x_i,\, r''_{t,i} \leq x_j \m \rank_{i,j}=\theta_0,\hat{L}_t, a_{t,i}]
        \sP\qty[\rank_{i,j}=\theta_0\m\hat{L}_t, a_{t,i}]. \numberthis{\label{eq: prob_swap}}
    \end{align*}
    By symmetry of $r''_t\in [\nu,\infty)^d$, we have
    \begin{align}
        \sP\qty[\rank_{i,j}=\theta_0\m\hat{L}_t, a_{t,i}]
        =\sP\qty[\rank_{i,i}=\theta_0\m\hat{L}_t, a_{t,i}]
        \label{max_prob}
    \end{align}
    for any $j$ suth that $\sigma_j\leq\sigma_i$.
    Then we have
    \begin{align*}
        1
        &=
        \sP\qty[\bigcup_{j: \sigma_j\leq\sigma_i}
        \{\rank_{i,j}=\theta_0\}\m\hat{L}_t, a_{t,i}]\\
        &=
        \sum_{j: \sigma_j\leq\sigma_i}\sP\qty[
        \rank_{i,j}=\theta_0\m\hat{L}_t, a_{t,i}]\\
        &=
        \sigma_i
        \sP\qty[
        \rank_{i,i}=\theta_0\m\hat{L}_t, a_{t,i}],
    \end{align*}
    which means that \eqref{max_prob} is equal to $1/\sigma_i$.
    Therefore, from \eqref{eq: prob_swap} we have 
    \begin{multline}
        \sP^*\qty[\bigcap\nolimits_{j:\sigma_j\leq\sigma_i} \left\{r''_{t,j} \leq x_j\right\}\m\hat{L}_t, a_{t,i},\theta=\theta_0] = \\
        \frac{1}{\sigma_i}\sum_{j:\sigma_j\leq\sigma_i}\sP\qty[\bigcap\nolimits_{i: \sigma_i\leq\sigma_i,i\notin\qty{j,i}} \left\{r''_{t,i} \leq x_i\right\},\, r''_{t,j} \leq x_i,\, r''_{t,i} \leq x_j \m \rank_{i,j}=\theta_0,\hat{L}_t, a_{t,i}].
        \label{eq: prob_swap2}
    \end{multline}
    By symmetry, each probability term in the RHS of \eqref{eq: prob_swap2} is equal. Therefore, we have
    \begin{equation*}
        \sP^*\qty[\bigcap\nolimits_{j:\sigma_j\leq\sigma_i} \left\{r''_{t,j} \leq x_j\right\}\m\hat{L}_t, a_{t,i},\theta=\theta_0] 
        = 
        \sP\qty[\bigcap\nolimits_{j:\sigma_j\leq\sigma_i} \left\{r''_{t,j} \leq x_j\right\}\m\rank_{i,i}=\theta_0,\hat{L}_t, a_{t,i}],
    \end{equation*}
    with which we immediately obtain
    \begin{multline}\label{eq:prob_sum}
        \frac{1}{m}\sum_{\theta_0\in[m]}\sP^*\qty[\bigcap\nolimits_{j:\sigma_j\leq\sigma_i} \left\{r''_{t,j} \leq x_j\right\}\m\hat{L}_t, a_{t,i},\theta=\theta_0] =\\ 
        \frac{1}{m}\sum_{\theta_0\in[m]}\sP\qty[\bigcap\nolimits_{j:\sigma_j\leq\sigma_i} \left\{r''_{t,j} \leq x_j\right\}\m\rank_{i,i}=\theta_0,\hat{L}_t, a_{t,i}].
    \end{multline}
    For the LHS of \eqref{eq:prob_sum}, since for any $\theta_0\in[m]$ we have $\sP[\theta=\theta_0]=1/m$, we have
    \begin{align*}
        &\frac{1}{m}\sum_{\theta_0\in[m]}\sP^*\qty[\bigcap\nolimits_{j:\sigma_j\leq\sigma_i} \left\{r''_{t,j} \leq x_j\right\}\m\hat{L}_t, a_{t,i},\theta=\theta_0]\\
        =&\sum_{\theta_0\in[m]}\sP^*\qty[\bigcap\nolimits_{j:\sigma_j\leq\sigma_i} \left\{r''_{t,j} \leq x_j\right\}\m\hat{L}_t, a_{t,i},\theta=\theta_0]\sP\qty[\theta=\theta_0]\\
        =&\sum_{\theta_0\in[m]}\sP^*\qty[\bigcap\nolimits_{j:\sigma_j\leq\sigma_i} \left\{r''_{t,j} \leq x_j\right\}\m\hat{L}_t, a_{t,i},\theta=\theta_0]\sP\qty[\theta=\theta_0\m\hat{L}_t, a_{t,i}]\\
        =&\sum_{\theta_0\in[m]}\sP^*\qty[\bigcap\nolimits_{j:\sigma_j\leq\sigma_i} \left\{r''_{t,j} \leq x_j\right\},\theta=\theta_0\m\hat{L}_t, a_{t,i}]\\
        =&\sP^*\qty[\bigcap\nolimits_{j:\sigma_j\leq\sigma_i} \left\{r''_{t,j} \leq x_j\right\}\m\hat{L}_t, a_{t,i}].\numberthis{\label{eq:prob_sum_LHS}}
    \end{align*}
    On the other hand, for any $\theta_0\in[m]$ and $\sigma_i\geq m$ we have
    \begin{equation}\label{eq:prob_event}
        \sP\qty(\mathcal{E}_{t,i})=\sP\qty[\rank_{i,i}\in[m]\m\,\hat{L}_t, a_{t,i}]=
        \sum_{\theta_0\in[m]}\sP\qty[\rank_{i,i}=\theta_0\m\,\hat{L}_t, a_{t,i}]=
        m/\sigma_i,
    \end{equation}
    and
    \begin{align*}
        \sP\qty[\rank_{i,i}=\theta_0\m\rank_{i,i}\in[m],\hat{L}_t, a_{t,i}] &=
        \frac{\sP\qty[\rank_{i,i}=\theta_0\m\,\hat{L}_t, a_{t,i}]}{\sP\qty[\rank_{i,i}\in[m]\m\,\hat{L}_t, a_{t,i}]}\\
        &=\frac{1/\sigma_i}{m/\sigma_i}=\frac{1}{m}.
    \end{align*}
    Then, for the RHS of \eqref{eq:prob_sum} we have
    \begin{align*}
        &\frac{1}{m}\sum_{\theta_0\in[m]}\sP\qty[\bigcap\nolimits_{j:\sigma_j\leq\sigma_i} \left\{r''_{t,j} \leq x_j\right\}\m\rank_{i,i}=\theta_0,\hat{L}_t, a_{t,i}]\\
        =&\sum_{\theta_0\in[m]}\sP\qty[\bigcap\nolimits_{j:\sigma_j\leq\sigma_i} \left\{r''_{t,j} \leq x_j\right\}\m\rank_{i,i}=\theta_0,\hat{L}_t, a_{t,i}]\sP\qty[\rank_{i,i}=\theta_0\m\rank_{i,i}\in[m],\hat{L}_t, a_{t,i}]\\
        =&\sum_{\theta_0\in[m]}\sP\qty[\bigcap\nolimits_{j:\sigma_j\leq\sigma_i} \left\{r''_{t,j} \leq x_j\right\},\rank_{i,i}=\theta_0\m\rank_{i,i}\in[m],\hat{L}_t, a_{t,i}]\\
        =&\sP\qty[\bigcap\nolimits_{j:\sigma_j\leq\sigma_i} \left\{r''_{t,j} \leq x_j\right\}\m\rank_{i,i}\in[m],\hat{L}_t, a_{t,i}].\numberthis{\label{eq:prob_sum_RHS}}
    \end{align*}
    Combining \eqref{eq:prob_sum}, \eqref{eq:prob_sum_LHS} and \eqref{eq:prob_sum_RHS}, we have
    \begin{equation*}
        \sP^*\qty[\bigcap\nolimits_{j:\sigma_j\leq\sigma_i} \left\{r''_{t,j} \leq x_j\right\}\m\hat{L}_t, a_{t,i}]
        =
        \sP\qty[\bigcap\nolimits_{j:\sigma_j\leq\sigma_i} \left\{r''_{t,j} \leq x_j\right\}\m\rank_{i,i}\in[m],\hat{L}_t, a_{t,i}],
    \end{equation*}
    which means that CGR samples $r''_t$ from the conditional distribution of $\mathcal{D}$ conditioned on $\qty{\rank_{i,i}\in[m]}$.
    Combining this fact and \eqref{eq:prob_event} with Lemma~\ref{lem:general_idea}, for $\sigma_i\geq m$ we have
    \begin{equation*}
        \E_{r''_t \sim \mathcal{D}|\mathcal{E}_{t,i}}[M_t|\hat{L}_t, a_{t,i}]=\frac{\sP\qty[\rank_{i,i}\in[m]\m\,\hat{L}_t, a_{t,i}]}{w_{t,i}}=\frac{m}{\sigma_i w_{t,i}}.
    \end{equation*}
    Note that for $i$ satisfying $\sigma_i< m$, the resampling method is the same as the original GR. For such $i$ we have
    \begin{equation*}
        \E_{r'_t\sim\mathcal{D}}[M_t|\hat{L}_t, a_{t,i}]=\frac{1}{w_{t,i}}.
    \end{equation*}
    Therefore, for any $i\in[d]$,  
    \begin{equation*}
        \widehat{w_{t,i}^{-1}}=\qty(\frac{\sigma_i}{m}\lor 1)M_{t,i}
    \end{equation*}
    serves as an unbiased estimator for $w_{t,i}^{-1}$. 
    Then, the expected number $M_t$ of resampling given $\hat{L}_t$ in CGR is bounded by 
    \begin{align*}
        \E_{r'_t\sim \mathcal{D},r''_t \sim \mathcal{D}|\mathcal{E}_{t,i}}[M_t|\hat{L}_t] 
        &=\E_{r'_t\sim \mathcal{D},r''_t \sim \mathcal{D}|\mathcal{E}_{t,i}}\qty[\max_{i:a_{t,i}=1,\sigma_i\leq m}M_{t,i} + \sum_{i:a_{t,i}=1,\sigma_i> m} M_{t,i} \m\hat{L}_t, a_t]\\
        &\leq\E_{r'_t\sim \mathcal{D},r''_t \sim \mathcal{D}|\mathcal{E}_{t,i}}\qty[\sum_{i=1}^d a_{t,i}M_{t,i} \m\hat{L}_t, a_t]\\
        &=\sum_{i=1}^d \sP[a_{t,i}=1|\hat{L}_t]\E_{r'_t\sim \mathcal{D},r''_t \sim \mathcal{D}|\mathcal{E}_{t,i}}[M_{t,i}|\hat{L}_t, a_{t,i}=1] \\
        &= \sum_{i=1}^m w_{t,i}\cdot\frac{1}{w_{t,i}}
        + \sum_{i=m+1}^d w_{t,i}\cdot\frac{m}{\sigma_i w_{t,i}} \\
        &\leq m+m\int_{m}^{d} \frac{1}{x} \dd x  \\
        &= m+m\log\qty(\frac{d}{m}).
    \end{align*}
\end{proof}

\paragraph{Average Complexity}
Now we analyze the average complexity of CGR, which can be expressed as 
\begin{equation}\label{eq:total_complexity}
    C_{\text{CGR}}=C_{\text{filter}}+\E_{r'_t\sim \mathcal{D},r''_t\sim\mathcal{D}|\mathcal{E}_{t,i}}\qty[M_t\m \hat{L}_{t}]\cdot C_{\text{resampling}},
\end{equation}
where $C_{\text{filter}}$ is the cost of scanning the base-arms and determining whether to include them in the set $U$ (Lines~\ref{line:scan_begin}--\ref{line:scan_end}), and $C_{\text{resampling}}$ is the cost of each resampling. For the former, the condition $\sigma_i > m$ that requires $O(d)$ is only evaluated when $a_i=1$. Then, we have
\begin{equation}\label{eq:complexity_filter}
    C_{\text{filter}}=
    d\cdot O(1)+\left\lVert a\right\rVert_1\cdot O(d)=
    d\cdot O(1)+m\cdot O(d)=O\qty(md). 
\end{equation}
For the resampling process, as shown in Algorithm~\ref{alg:CGR}, base-arms in $U$ and those not in $U$ are resampled differently, since the former involves an additional value-swapping operation (Lines~\ref{line:swap_begin}--\ref{line:swap_end}). However, this operation does not change the order of the resampling cost, which remains $C_{\text{resampling}}=O(d)$ in both cases. Combining \eqref{eq:resampling_bound}, \eqref{eq:total_complexity} and \eqref{eq:complexity_filter}, we have
\begin{equation*}
    C_{\text{CGR}}= O(md)+\qty(m+m\log\qty(d/m))\cdot O(d)=O\qty(md\qty(\log\qty(d/m)+1)).
\end{equation*}
\begin{remark}
    In this paper, though we only analyze the regret bound of FTPL with the original GR, the analysis of FTPL with CGR is similar, as we only need to replace the expression of expectation of the estimator $\widehat{{w_{t,i}^{-1}}}^2$ in \eqref{eq:square_expectation}. In fact, the regret bound of FTPL with CGR can attain a slightly better regret bound compared with the one with the original GR. This is because the variance of  $\widehat{{w_{t,i}^{-1}}}$ becomes
    \begin{equation*}
        \mathrm{Var}[\widehat{w_{t,i}^{-1}}|\hat{L}_t, a_{t,i}]
        =
        \begin{cases}
            \frac{1}{w_{t,i}^2}-\frac{1}{w_{t,i}}&\mbox{(original GR)},\\
            \frac{1}{w_{t,i}^2}-\frac{1}{\sP(\mathcal{E}_{t,i})w_{t,i}}&\mbox{(CGR)},
        \end{cases}
    \end{equation*}
    where the latter is no larger than the former.
\end{remark}

\section{Proofs for regret decomposition}\label{sec:proof_regret_decomposition}
In this section, we provide the proof of Lemma~\ref{lem:regret_decomposition}. Firstly, similarly to Lemma~3 in \citet{pmlr-v201-honda23a}, we prove the general framework of the regret decomposition that can be applied to general distributions.
\begin{lemma}\label{lem:regret_decomposition_general}
    \begin{equation*}
        \R(T) \leq \sum_{t=1}^{T} \mathbb{E} \qty[ \left\langle \hat{\ell}_t, w_t - w_{t+1} \right\rangle ] 
        + \frac{\E_{r_1 \sim \dis} \qty[ a_1^\top r_1 ]}{\eta}.
    \end{equation*}
\end{lemma}

\begin{proof}
    Let us consider random variable $r \in [0, \infty)^d$ that independently follows Fréchet distribution $\Fdis$ or Pareto distribution $\Pdis$, and is independent from the randomness $\{\ell_t, r_t\}_{t=1}^{T}$ of the environment and the policy.
    Define 
    $u_t = \argmin_{w \in \Delta _d} \left\langle \eta \hat{L}_t - r, w \right\rangle$,
    where $\Delta_d = \{ p \in [0,1]^d : \sum_{i \in [d]} 1\leq p_i\leq m \}$. Then, since $r_t$ and $r$ are identically distributed given $\hat{L}_t$, we have
    \begin{equation}\label{eq:decom_expectation}
    \mathbb{E} [ u_t | \hat{L}_t ] = w_t, \quad \mathbb{E} [ \langle r, u_t \rangle | \hat{L}_t ] = \mathbb{E} [ a_t^\top r | \hat{L}_t ].
    \end{equation}

Denote the optimal action as $a^*$. Recalling $\hat{L}_t = \sum_{s=1}^{t} \hat{\ell}_s$, we have
\begin{align*}
\sum_{t=1}^{T} \left\langle \hat{\ell}_t, a^* \right\rangle &= \left\langle \hat{L}_{T+1}, a^* \right\rangle\\ 
&= \left\langle \hat{L}_{T+1} - \frac{1}{\eta} r, a^* \right\rangle + \frac{1}{\eta} \langle r, a^* \rangle \\
&\geq \left\langle \hat{L}_{T+1} - \frac{1}{\eta} r, u_{T+1} \right\rangle + \frac{1}{\eta} \langle r, a^* \rangle \\
&= \left\langle \hat{L}_{T} - \frac{1}{\eta} r, u_{T+1} \right\rangle + \left\langle \hat{\ell}_T, u_{T+1} \right\rangle + \frac{1}{\eta} \langle r, a^* \rangle \\
&\geq \left\langle \hat{L}_{T} - \frac{1}{\eta} r, u_T \right\rangle + \left\langle \hat{\ell}_T, u_{T+1} \right\rangle+ \frac{1}{\eta} \langle r, a^* \rangle
\end{align*}
and recursively applying this relation, we obtain
\begin{equation*}
\sum_{t=1}^{T} \left\langle \hat{\ell}_t, a^* \right\rangle 
\geq \left\langle -\frac{1}{\eta} r, u_1 \right\rangle  + \sum_{t=1}^{T} \left\langle \hat{\ell}_t, u_{t+1} \right\rangle+ \frac{1}{\eta} \langle r, a^* \rangle
\end{equation*}
and therefore
\begin{equation*}
\sum_{t=1}^{T} \left\langle \hat{\ell}_t, u_t - a^* \right\rangle 
\leq \frac{1}{\eta} \left\langle r, u_1 - a^* \right\rangle 
+ \sum_{t=1}^{T} \left\langle \hat{\ell}_t, u_t - u_{t+1} \right\rangle .
\end{equation*}

By using \eqref{eq:decom_expectation} and taking the expectation with respect to $r$ we obtain
\begin{align*}
\sum_{t=1}^{T} \left\langle \hat{\ell}_t, w_t - a^* \right\rangle 
&\leq \frac{1}{\eta} \mathbb{E}_{r \sim \dis} \left[ \left\langle r, u_1 - a^* \right\rangle \right] + \sum_{t=1}^{T} \left\langle \hat{\ell}_t, w_t - w_{t+1} \right\rangle \\
&\leq \frac{1}{\eta} \mathbb{E}_{r_1 \sim \dis} \left[ a_1^\top r_1 \right] 
+ \sum_{t=1}^{T} \left\langle \hat{\ell}_t, w_t - w_{t+1} \right\rangle.
\end{align*}
\end{proof}
For Fr\'{e}chet and Pareto distributions, we bound $\mathbb{E}_{r_1 \sim \dis} \left[ a_1^\top r_1 \right]$ in the following lemma.
\begin{lemma}\label{lem:penalty_bound}
    For $\dis\in\qty{\Pdis,\Fdis}$ and $\alpha>1$, we have 
    \begin{equation*}
        \mathbb{E}_{r_1 \sim \dis}\left[ a_1^\top r_1 \right]\leq
        \begin{cases}
            \qty(\frac{\alpha}{\alpha-1}m^{1-\frac{1}{\alpha}}+\Gamma\qty(1-\frac{1}{\alpha}))\qty(d+1)^{\frac{1}{\alpha}} & \text{if } \dis = \Pdis \\
            \qty(\frac{\alpha}{\alpha-1}m^{1-\frac{1}{\alpha}}+\Gamma\qty(1-\frac{1}{\alpha}))\qty(d+1)^{\frac{1}{\alpha}}+m & \text{if } \dis = \Fdis.
        \end{cases}
    \end{equation*}
\end{lemma}

\begin{proof}
    Let $r_k^*$ be the $k$-th largest perturbation among $r_{1,1},r_{1,2},\cdots,r_{1,d}$ i.i.d. sampled from $\mathcal{D}_{\alpha}$ for $k\in[d]$.
    Then, we have
    \begin{equation}\label{eq:penalty_order_decom}
        \E_{r \sim \dis}\qty[a_1^\top r_1]\leq
        \E_{r \sim \dis}\qty[\sum_{k=1}^m r_k^*]
        \leq \sum_{k=1}^m \E_{r \sim \dis}\qty[r_k^*].
    \end{equation}
    \item \paragraph{Pareto Distribution}
    If $\dis=\Pdis$, by Lemma~\ref{lem:pareto_order_statistics}, we obtain
    \begin{equation}\label{eq:penalty_order_decom2}
        \sum_{k=1}^m \E_{r \sim \Pdis}\qty[r_k^*]
        \leq
        \sum_{k=1}^m \frac{\Gamma\qty(d+1)\Gamma\qty(d-k-\frac{1}{\alpha}+1)}{\Gamma\qty(d-k+1)\Gamma\qty(d-\frac{1}{\alpha}+1)}.
    \end{equation}
    For $k=m=d$, we have
    \begin{align*}
        \frac{\Gamma\qty(d+1)\Gamma\qty(d-k-\frac{1}{\alpha}+1)}{\Gamma\qty(d-k+1)\Gamma\qty(d-\frac{1}{\alpha}+1)}&=
        \frac{\Gamma\qty(d+1)\Gamma\qty(1-\frac{1}{\alpha})}{\Gamma\qty(d-\frac{1}{\alpha}+1)}\\
        &\leq
        \Gamma\qty(1-\frac{1}{\alpha})\qty(d+1)^{\frac{1}{\alpha}},\numberthis\label{eq:penalty_order_kmd}
    \end{align*}
    where the last inequality follows from Gautschi's inequality in Lemma~\ref{lem:Gautschi}.
    Similarly, for $k\in[m]$ and $k<d$, 
    by Gautschi's inequality, we have
    \begin{equation}\label{eq:penalty_Gautschi}
        \frac{\Gamma\qty(d+1)\Gamma\qty(d-k-\frac{1}{\alpha}+1)}{\Gamma\qty(d-k+1)\Gamma\qty(d-\frac{1}{\alpha}+1)} \leq \qty(\frac{d+1}{d-k})^{\frac{1}{\alpha}}.
    \end{equation}
    By combining \eqref{eq:penalty_order_decom}, \eqref{eq:penalty_order_decom2}, \eqref{eq:penalty_order_kmd} and \eqref{eq:penalty_Gautschi}, we have
    \begin{align*}
        \E_{r \sim \Pdis}\qty[a_1^\top r_1]
        &\leq 
        \sum_{k=1}^m \E_{r \sim \Pdis}\qty[r_k^*]\\
        &\leq \Gamma\qty(1-\frac{1}{\alpha})\qty(d+1)^{\frac{1}{\alpha}}+
        \sum_{k=1}^{m\land(d-1)} \qty(\frac{d+1}{d-k})^{\frac{1}{\alpha}}\\
        &<
        \Gamma\qty(1-\frac{1}{\alpha})\qty(d+1)^{\frac{1}{\alpha}}+\sum_{k=1}^m \qty(\frac{d+1}{k})^{\frac{1}{\alpha}}\\
        &\leq \qty(\Gamma\qty(1-\frac{1}{\alpha})+1+\int_{1}^m x^{-\frac{1}{\alpha}} \dd x)\qty(d+1)^{\frac{1}{\alpha}}\\
        &= \qty(\Gamma\qty(1-\frac{1}{\alpha})+1+\frac{\alpha}{\alpha-1}x^{1-\frac{1}{\alpha}}\bigg|_1^m)\qty(d+1)^{\frac{1}{\alpha}}\\
        &= \qty(\frac{\alpha}{\alpha-1}m^{1-\frac{1}{\alpha}}+\Gamma\qty(1-\frac{1}{\alpha})-\frac{1}{\alpha-1})\qty(d+1)^{\frac{1}{\alpha}}\\
        &< \qty(\frac{\alpha}{\alpha-1}m^{1-\frac{1}{\alpha}}+\Gamma\qty(1-\frac{1}{\alpha}))\qty(d+1)^{\frac{1}{\alpha}}.\numberthis\label{eq:penalty_pareto}
    \end{align*}
    \item \paragraph{Fr\'{e}chet Distribution}
    If $\dis=\Fdis$, by combining \eqref{eq:penalty_order_decom}, Lemma~\ref{lem:penalty_pareto_frechet} and \eqref{eq:penalty_pareto} we have
    \begin{align*}
        \E_{r \sim \Fdis}\qty[a_1^\top r_1]
        &\leq \sum_{k=1}^m \E_{r \sim \Fdis}\qty[r_k^*]\\
        &\leq \sum_{k=1}^m \E_{r \sim \Pdis}\qty[r_k^*]+1\\
        &< \qty(\frac{\alpha}{\alpha-1}m^{1-\frac{1}{\alpha}}+\Gamma\qty(1-\frac{1}{\alpha}))\qty(d+1)^{\frac{1}{\alpha}} + m.
    \end{align*}
\end{proof}

\section{Analysis on Stability Term}\label{appendix:stab}

\subsection{Proof for Monotonicity}

\begin{lemma}\label{lem:both_increasing}
    Let $\psi (x):[\nu-\lambda_i,\infty)\to\mathbb{R}$ denote a non-negative function that is independent of $\lambda_j$. 
    If $j\neq i$ and $F(x)$ is the cumulative distribution function of Fr\'{e}chet or Pareto distributions, then
    \begin{equation*}
        \frac{\int_{\nu-(\lambda_i\land\lambda_j)}^\infty \psi(z)F(z+\lambda_j)/(z+\lambda_i) \dd z}{\int_{\nu-(\lambda_i\land\lambda_j)}^\infty \psi(z)F(z+\lambda_j) \dd z}
    \end{equation*}
    is monotonically increasing in $\lambda_j$.
\end{lemma}

\begin{proof}
    Let $$N(\lambda_j)=\int_{\nu-(\lambda_i\land\lambda_j)}^\infty \frac{1}{z+\lambda_i}\psi(z)F(z+\lambda_j) \dd z,\quad D(\lambda_j)=\int_{\nu-(\lambda_i\land\lambda_j)}^\infty \psi(z)F(z+\lambda_j) \dd z.$$
    The derivative of $N(\lambda_j)/D(\lambda_j)$ with respect to $\lambda_j$ is expressed as
    \begin{equation*}
        \frac{d}{d \lambda_j}N(\lambda_j)/D(\lambda_j)=\frac{N'(\lambda_j)D(\lambda_j)-N(\lambda_j)D'(\lambda_j)}{(D(\lambda_j))^2}.
    \end{equation*}
    If $\lambda_i>\lambda_j$, we have
    \begin{align*}
        N'(\lambda_j)&=\frac{\partial}{\partial \lambda_j}\int_{\nu-\lambda_j}^\infty \frac{1}{z+\lambda_i}\psi(z)F(z+\lambda_j)\dd z\\
        &=\int_{\nu-\lambda_j}^\infty \frac{1}{z+\lambda_i}\psi(z)f(z+\lambda_j)\dd z + \frac{1}{(\nu-\lambda_j)+\lambda_i}\psi(\nu-\lambda_j)F((\nu-\lambda_j)+\lambda_j)\\
        &=\int_{\nu-\lambda_j}^\infty \frac{1}{z+\lambda_i}\psi(z)f(z+\lambda_j)\dd z,
    \end{align*}
    where the last equality holds since $F(\nu)=0$. On the other hand, if $\lambda_i\leq\lambda_j$, we have
    \begin{equation*}
        N'(\lambda_j)=\frac{\partial}{\partial \lambda_j}\int_{\nu-\lambda_i}^\infty \frac{1}{z+\lambda_i}\psi(z)F(z+\lambda_j)\dd z=\int_{\nu-\lambda_i}^\infty \frac{1}{z+\lambda_i}\psi(z)f(z+\lambda_j)\dd z.
    \end{equation*}
    In both cases, we have
    \begin{equation*}
        N'(\lambda_j)=\int_{\nu-(\lambda_i\land\lambda_j)}^\infty \frac{1}{z+\lambda_i}\psi(z)f(z+\lambda_j)\dd z.
    \end{equation*}
    Similarly, we have
    \begin{equation*}
        D'(\lambda_j)=\int_{\nu-(\lambda_i\land\lambda_j)}^\infty \psi(z)f(z+\lambda_j)\dd z,
    \end{equation*}

    Next, we divide the proof into two cases.
    \item \paragraph{Fr\'{e}chet Distribution}
    When $F(x)$ is the cumulative distribution function of Fr\'{e}chet distribution, we define $\widetilde{\psi}(x)=\psi(x)e^{-1/(x+\lambda_j)^\alpha}$. 
    Under this definition, we have
    \begin{align*}
        N'(\lambda_j)D(\lambda_j)&=\iint_{z,w \geq -(\lambda_i\land\lambda_j)} \frac{\psi(z)\psi(w)f(z+\lambda_j)F(w+\lambda_j)}{(z+\lambda_i)} \dd z \dd w\\
        &=\alpha\iint_{z,w \geq -(\lambda_i\land\lambda_j)} \frac{\widetilde{\psi}(z)\widetilde{\psi}(w)}{(z+\lambda_i)(z+\lambda_j)^{\alpha+1}}\dd z \dd w\\
        &=\frac{\alpha}{2}\iint_{z,w \geq -(\lambda_i\land\lambda_j)} \widetilde{\psi}(z)\widetilde{\psi}(w)\qty(\frac{1}{(z+\lambda_i)(z+\lambda_j)^{\alpha+1}}+\frac{1}{(w+\lambda_i)(w+\lambda_j)^{\alpha+1}})\dd z \dd w.
    \end{align*}
    and
    \begin{align*}
        N(\lambda_j)D'(\lambda_j)&=\iint_{z,w \geq -(\lambda_i\land\lambda_j)} \frac{\psi(z)\psi(w)F(z+\lambda_j)f(w+\lambda_j)}{(z+\lambda_i)} \dd z \dd w\\
        &=\alpha\iint_{z,w \geq -(\lambda_i\land\lambda_j)} \frac{\widetilde{\psi}(z)\widetilde{\psi}(w)}{(z+\lambda_i)(w+\lambda_j)^{\alpha+1}}\dd z \dd w\\
        &=\frac{\alpha}{2}\iint_{z,w \geq -(\lambda_i\land\lambda_j)} \widetilde{\psi}(z)\widetilde{\psi}(w)\qty(\frac{1}{(z+\lambda_i)(w+\lambda_j)^{\alpha+1}}+\frac{1}{(w+\lambda_i)(z+\lambda_j)^{\alpha+1}})\dd z \dd w.
    \end{align*}
    Here, by an elementary calculation we can see
    \begin{align*}
        &\frac{1}{(z+\lambda_i)(z+\lambda_j)^{\alpha+1}}+\frac{1}{(w+\lambda_i)(w+\lambda_j)^{\alpha+1}}-\frac{1}{(z+\lambda_i)(w+\lambda_j)^{\alpha+1}}-\frac{1}{(w+\lambda_i)(z+\lambda_j)^{\alpha+1}}\\
        &=\frac{w-z}{(z+\lambda_i)(w+\lambda_i)}\qty(\frac{1}{(z+\lambda_j)^{\alpha+1}}-\frac{1}{(w+\lambda_j)^{\alpha+1}})\\
        &=\frac{(w-z)\qty(\qty(w+\lambda_j)^{\alpha+1}-\qty(z+\lambda_j)^{\alpha+1})}{(z+\lambda_i)(w+\lambda_i)(z+\lambda_j)^{\alpha+1}(w+\lambda_j)^{\alpha+1}}\geq 0,
    \end{align*}
    where the last inequality holds since $h(x)=x^{\alpha+1}$ is monotonically increasing in $[0,+\infty)$ for $\alpha>0$. Therefore, when $F(x)$ is the cumulative distribution function of Fr\'{e}chet distribution, we have $\frac{d}{d \lambda_j}N(\lambda_j)/D(\lambda_j)\geq 0$, which implies that $N(\lambda_j)/D(\lambda_j)$ is monotonically increasing in $\lambda_j$.

    \item \paragraph{Pareto Distribution}
    When $F(x)$ is the cumulative distribution function of Pareto distribution, we have
    \begin{align*}
        N'(\lambda_j)D(\lambda_j)&=\iint_{z,w \geq 1-(\lambda_i\land\lambda_j)} \frac{\psi(z)\psi(w)f(z+\lambda_j)F(w+\lambda_j)}{(z+\lambda_i)} \dd z \dd w\\
        &=\alpha\iint_{z,w \geq 1-(\lambda_i\land\lambda_j)} \frac{\psi(z)\psi(w)(1-(w+\lambda_j)^{-\alpha})}{(z+\lambda_i)(z+\lambda_j)^{\alpha+1}}\dd z \dd w\\
        &=\frac{\alpha}{2}\iint_{z,w \geq 1-(\lambda_i\land\lambda_j)} \psi(z)\psi(w)\qty(\frac{(1-(w+\lambda_j)^{-\alpha})}{(z+\lambda_i)(z+\lambda_j)^{\alpha+1}}+\frac{(1-(z+\lambda_j)^{-\alpha})}{(w+\lambda_i)(w+\lambda_j)^{\alpha+1}})\dd z \dd w,
    \end{align*}
    and
    \begin{align*}
        N(\lambda_j)D'(\lambda_j)&=\iint_{z,w \geq 1-(\lambda_i\land\lambda_j)} \frac{\psi(z)\psi(w)f(z+\lambda_j)F(w+\lambda_j)}{(z+\lambda_i)} \dd z \dd w\\
        &=\alpha\iint_{z,w \geq 1-(\lambda_i\land\lambda_j)} \frac{\psi(z)\psi(w)(1-(z+\lambda_j)^{-\alpha})}{(z+\lambda_i)(w+\lambda_j)^{\alpha+1}}\dd z \dd w\\
        &=\frac{\alpha}{2}\iint_{z,w \geq 1-(\lambda_i\land\lambda_j)} \psi(z)\psi(w)\qty(\frac{(1-(z+\lambda_j)^{-\alpha})}{(z+\lambda_i)(w+\lambda_j)^{\alpha+1}}+\frac{(1-(w+\lambda_j)^{-\alpha})}{(w+\lambda_i)(z+\lambda_j)^{\alpha+1}})\dd z \dd w.
    \end{align*}
    Here, by an elementary calculation we can see
    \begin{align*}
        &\frac{(1-(w+\lambda_j)^{-\alpha})}{(z+\lambda_i)(z+\lambda_j)^{\alpha+1}}+\frac{(1-(z+\lambda_j)^{-\alpha})}{(w+\lambda_i)(w+\lambda_j)^{\alpha+1}}-\frac{(1-(z+\lambda_j)^{-\alpha})}{(z+\lambda_i)(w+\lambda_j)^{\alpha+1}}-\frac{(1-(w+\lambda_j)^{-\alpha})}{(w+\lambda_i)(z+\lambda_j)^{\alpha+1}}\\
        &=\frac{w-z}{(z+\lambda_i)(w+\lambda_i)}\qty(\frac{1-(w+\lambda_j)^{-\alpha}}{(z+\lambda_j)^{\alpha+1}}-\frac{1-(z+\lambda_j)^{-\alpha}}{(w+\lambda_j)^{\alpha+1}})\\
        &=\frac{w-z}{(z+\lambda_i)(w+\lambda_i)(z+\lambda_j)^{\alpha+1}(w+\lambda_j)^{\alpha+1}}\qty(\qty((w+\lambda_j)^{\alpha+1}-(w+\lambda_j))-\qty((z+\lambda_j)^{\alpha+1}-(z+\lambda_j)))\geq 0,
    \end{align*}
    where the last inequality holds because $h(x)=x^{\alpha+1}-x$ is monotonically increasing in $[1,+\infty)$ for $\alpha>0$. Therefore, we have $\frac{d}{d \lambda_j}N(\lambda_j)/D(\lambda_j)\geq 0$, which concludes the proof.
\end{proof}

\begin{lemma}\label{lem:bgaf}
    Let $a,b>0$, $f(x),g(x)>0$, where $x\in\mathbb{R}$. If both $f(x)$ and $g(x)/f(x)$ are monotonically increasing in $x$, then for any $x_1<x_2$, we have 
    \begin{equation*}
        \frac{b+g(x_1)}{a+f(x_1)}\leq\frac{b}{a}\lor\frac{b+g(x_2)}{a+f(x_2)}.
    \end{equation*}
    Provided that $\lim_{x\to\infty}\qty(b+g(x))/\qty(a+f(x))$ exists, for any $x_0\in\mathbb{R}$ we have
    \begin{equation*}
        \frac{b+g(x_0)}{a+f(x_0)}\leq\frac{b}{a}\lor\lim_{x\to\infty}\frac{b+g(x)}{a+f(x)}.
    \end{equation*}
\end{lemma}

\begin{proof}
    According to the assumption, we have 
    \begin{equation*}
        f(x_1)\leq f(x_2) \text{ and } \frac{g(x_1)}{f(x_1)}\leq \frac{g(x_2)}{f(x_2)}.
    \end{equation*}
    If $b/a>g(x_2)/f(x_2)$, then we have
        \begin{equation*}
            \frac{b+g(x_1)}{a+f(x_1)}\leq\frac{b}{a}\lor\frac{g(x_1)}{f(x_1)}\leq\frac{b}{a}\lor\frac{g(x_2)}{f(x_2)}\leq\frac{b}{a}\leq\frac{b}{a}\lor\frac{b+g(x_2)}{a+f(x_2)}.
        \end{equation*}
    On the other hand, if $b/a\leq g(x_2)/f(x_2)$, then we have
        \begin{equation*}
            \frac{b+g(x_1)}{a+f(x_1)}=\frac{b+f(x_1)\frac{g(x_1)}{f(x_1)}}{a+f(x_1)}\leq\frac{b+f(x_1)\frac{g(x_2)}{f(x_2)}}{a+f(x_1)}.
        \end{equation*}
        Let $h(z)=\qty(b+\frac{g(x_2)}{f(x_2)}z)/\qty(a+z)$, where $z\in[f(x_1),f(x_2)]$. Then, we have
        \begin{align*}
            h'(z)&=\frac{\frac{g(x_2)}{f(x_2)}\qty(a+z)-\qty(b+\frac{g(x_2)}{f(x_2)}z)}{\qty(a+z)^2}\\
            &=\frac{\frac{g(x_2)}{f(x_2)}a-b}{\qty(a+z)^2}
            =\frac{a\qty(\frac{g(x_2)}{f(x_2)}-\frac{b}{a})}{\qty(a+z)^2}\geq 0,
        \end{align*}
        which means that $h(z)$ is monotonically increasing in $[f(x_1),f(x_2)]$. Therefore, we have
        \begin{equation*}
            \frac{b+f(x_1)\frac{g(x_2)}{f(x_2)}}{a+f(x_1)}\leq\frac{b+f(x_2)\frac{g(x_2)}{f(x_2)}}{a+f(x_2)}=\frac{b+g(x_2)}{a+f(x_2)}\leq\frac{b}{a}\lor\frac{b+g(x_2)}{a+f(x_2)}.
        \end{equation*}
    Combining both cases, we have
    \begin{equation*}
        \frac{b+g(x_1)}{a+f(x_1)}\leq\frac{b}{a}\lor\frac{b+g(x_2)}{a+f(x_2)}
    \end{equation*}
    for any $x_1<x_2$. If $\lim_{x\to\infty}\qty(b+g(x))/\qty(a+f(x))$ exists, the result for the infinite case follows directly by taking the limit of $x_2\to\infty$.
\end{proof}

\begin{lemma}\label{lem:important_inequality}
    For $\wm>1$ and $j\in\B\setminus\qty{i}$, let $\lambda'\in\mathbb{R}^d$ be such that $\lambda'_j\geq \lambda_j$ and $\lambda_k'=\lambda_k$ for all $k\neq j$.
    Then, we have
    \begin{equation*}
        \frac{J_i(\lambda;\dis,\wm,\B)}{\phi_i(\lambda;\dis,\wm,\B)}\leq\frac{J_i(\lambda';\dis,\wm-1,\B\setminus\qty{j})}{\phi_i(\lambda';\dis,\wm-1,\B\setminus\qty{j})}\lor\frac{J_i(\lambda';\dis,\wm,\B)}{\phi_i(\lambda';\dis,\wm,\B)}.
    \end{equation*}
\end{lemma}

\begin{proof}
    For $\theta\leq\wm$, we define
    \begin{equation*}
        \mathcal{E}_{i,\theta,j}^\B=\qty{\rank\qty(i,r-\lambda;\B)=\theta\leq\rank\qty(j,r-\lambda;\B)},
    \end{equation*}
    where the probability can be expressed as
    \begin{align*}
        \phi_{i,\theta,j}(\lambda;\dis,\B) &\coloneqq
        \sP_{r=\qty(r_1,\dots,r_d) \sim \mathcal{D}}\qty{\rank\qty(i,r-\lambda;\B)=\theta\leq\rank\qty(j,r-\lambda;\B)}\\
        &=\int_{\nu-(\lambda_i\land\lambda_j)}^\infty f(z+\lambda_i)F(z+\lambda_j)\sum_{\bm{v}\in\mathcal{S}_{i,\theta,j}^\B}\qty(\prod_{k:v_k=1}\qty(1-F(z+\lambda_k))\prod_{k:v_k=0,k\in\B\setminus\qty{i}}F(z+\lambda_k)) \dd z.\numberthis\label{eq:prob_k}
    \end{align*}
    Here, $\mathcal{S}_{i,\theta,j}^\B=\qty{\bm{v}\in\qty{0,1}^d:\normv_1=\theta-1,v_i=v_j=0,\text{ and }v_k=0\text{ for all }k\notin\B}$.
    Corresponding to this, we define
    \begin{multline}\label{eq:prob_k_J}
        J_{i,\theta,j}(\lambda;\dis,\B)\coloneqq\\
        \int_{\nu-(\lambda_i\land\lambda_j)}^\infty \frac{f(z+\lambda_i)}{z+\lambda_i}\sum_{\bm{v}\in\mathcal{S}_{i,\theta,j}^\B}\qty(\prod_{k:v_k=1}\qty(1-F(z+\lambda_k))\prod_{k:v_k=0,k\in\B\setminus\qty{i}}F(z+\lambda_k)) \dd z.
    \end{multline}
    Considering the event
    \begin{equation*}
        \widetilde{\mathcal{E}}_{i,j}=\qty{\rank\qty(i,r_t-\lambda_t;\B\setminus\qty{j})\leq\wm-1},
    \end{equation*}
    we have
    \begin{equation*}
        \sP_{r=\qty(r_1,\dots,r_d) \sim \mathcal{D}}\qty(\widetilde{\mathcal{E}}_{i,j})=\phi_i(\lambda;\dis,\wm-1,\B\setminus\qty{j}).
    \end{equation*}
    By definition, we can see that
    \begin{equation*}
        \qty{\rank\qty(i,r_t-\lambda_t;\B)\leq \wm}=\widetilde{\mathcal{E}}_{i,j}\cup\mathcal{E}_{i,\wm,j}^\B,
        \quad
        \widetilde{\mathcal{E}}_{i,j}\cap\mathcal{E}_{i,\wm,j}^\B=\varnothing,
    \end{equation*}
    Therefore, we can docompose the expression of the probability as

    \begin{equation}\label{eq:phi_decomose}
        \phi_i(\lambda;\dis,\wm,\B)=
        \phi_i(\lambda;\dis,\wm-1,\B\setminus\qty{j})+
        \phi_{i,\wm,j}(\lambda;\dis,\B).
    \end{equation}
    Similarly, by definition we have
    \begin{equation}\label{eq:J_decomose}
        J_i(\lambda;\dis,\wm,\B)=J_i(\lambda;\dis,\wm-1,\B\setminus\qty{j})+
        J_{i,\wm,j}(\lambda;\dis,\B).
    \end{equation}
    Consider the expression of 
    $\phi_{i,\theta,j}(\lambda;\dis,\wm,\B)$ and $J_{i,\theta,j}(\lambda;\dis,\wm,\B)$ respectively given by \eqref{eq:prob_k} and \eqref{eq:prob_k_J}. 
    By Lemma~\ref{lem:both_increasing}, it follows that
    \begin{equation*}
        \frac{J_{i,\wm,j}(\lambda;\dis,\B)}{\phi_{i,\wm,j}(\lambda;\dis,\B)}\leq
        \frac{J_{i,\wm,j}(\lambda';\dis,\B)}{\phi_{i,\wm,j}(\lambda';\dis,\B)}
    \end{equation*}
    because of the monotonic increase in $\lambda_j$. 
    Then, since both $\phi_i(\lambda;\dis,\wm-1,\B\setminus\qty{j})$ and $J_i(\lambda;\dis,\wm-1,\B\setminus\qty{j})$ do not depend on $\lambda_j$, we have
    \begin{align*}
        \frac{J_i(\lambda;\dis,\wm,\B)}{\phi_i(\lambda;\dis,\wm,\B)}
        &=\frac{J_i(\lambda;\dis,\wm-1,\B\setminus\qty{j})+J_{i,\wm,j}(\lambda;\dis,\B)}{\phi_i(\lambda;\dis,\wm-1,\B\setminus\qty{j})+\phi_{i,\wm,j}(\lambda;\dis,\B)}\\
        &\leq \frac{J_i(\lambda;\dis,\wm-1,\B\setminus\qty{j})}{\phi_i(\lambda;\dis,\wm-1,\B\setminus\qty{j})}\lor\frac{J_i(\lambda';\dis,\wm,\B)}{\phi_i(\lambda';\dis,\wm,\B)}\numberthis\label{eq:apply_bgaf}\\
        &=\frac{J_i(\lambda';\dis,\wm-1,\B\setminus\qty{j})}{\phi_i(\lambda';\dis,\wm-1,\B\setminus\qty{j})}\lor\frac{J_i(\lambda';\dis,\wm,\B)}{\phi_i(\lambda';\dis,\wm,\B)},
    \end{align*}
    where the inequality \eqref{eq:apply_bgaf} holds by recalling that the same relation as \eqref{eq:phi_decomose} and \eqref{eq:J_decomose} holds for $\lambda_j'$ in Lemma~\ref{lem:bgaf}.
\end{proof}

The following lemma presents a special case of Lemma~\ref{lem:important_inequality}, where we take $\lambda'_j\to\infty$. 
\begin{lemma}\label{lem:infty}
    Let $\wm\geq 2$. If $\wm<\left\lvert\B\right\rvert$, we have
    \begin{equation}\label{eq:important2}
        \frac{J_i(\lambda;\dis,\wm,\B)}{\phi_i(\lambda;\dis,\wm,\B)}\leq\frac{J_i(\lambda;\dis,\wm-1,\B\setminus\qty{j})}{\phi_i(\lambda;\dis,\wm-1,\B\setminus\qty{j})}\lor\frac{J_i(\lambda;\dis,\wm,\B\setminus\qty{j})}{\phi_i(\lambda;\dis,\wm,\B\setminus\qty{j})}.
    \end{equation}
    If $\wm=\left\lvert\B\right\rvert$, we have
    \begin{equation}\label{eq:important3}
         \frac{J_i(\lambda;\dis,\wm,\B)}{\phi_i(\lambda;\dis,\wm,\B)}\leq\frac{J_i(\lambda;\dis,\wm-1,\B\setminus\qty{j})}{\phi_i(\lambda;\dis,\wm-1,\B\setminus\qty{j})}.
    \end{equation}
\end{lemma}

\begin{proof}
    \item \paragraph{Inequality \eqref{eq:important2}}
    Recall that 
    \begin{multline*}
        \phi_i(\lambda;\dis,\wm,\B)=\\
        \sum_{\theta=1}^{m}\int_{\nu-\widetilde{\lambda}_\theta}^\infty f(z+\lambda_i)\sum_{\bm{v}\in\mathcal{S}_{i,\theta}^\B}\qty(\prod_{k:v_k=1}\qty(1-F(z+\lambda_k))\prod_{k:v_k=0,k\in\B\setminus\qty{i}}F(z+\lambda_k)) \dd z,
    \end{multline*}
    where $\mathcal{S}_{i,\theta}^\B=\qty{\bm{v}\in\qty{0,1}^d:\normv_1=\theta-1,v_i=0,\text{ and }v_k=0\text{ for all }k\notin\B}$. 
    Note that
    \begin{equation*}
        \mathcal{S}_{i,\theta}^\B\setminus\mathcal{S}_{i,\theta,j}^\B=\qty{\bm{v}\in\qty{0,1}^d:\normv_1=\theta-1,v_i=0,v_j=1,\text{ and }v_k=0\text{ for all }k\notin\B},
    \end{equation*}
    where $\mathcal{S}_{i,\theta,j}^\B=\qty{\bm{v}\in\qty{0,1}^d:\normv_1=\theta-1,v_i=v_j=0,\text{ and }v_k=0\text{ for all }k\notin\B}$. Then, $\phi_i(\lambda;\dis,\wm,\B)$ can be rewritten as
    \begin{multline}\label{eq:to_limit}
        \phi_i(\lambda;\dis,\wm,\B)=\\
        \sum_{\theta=1}^{\wm}\int_{\nu-\widetilde{\lambda}_\theta}^\infty f(z+\lambda_i)(1-F(z+\lambda_j))\sum_{\bm{v}\in\mathcal{S}_{i,\theta}^\B\setminus\mathcal{S}_{i,\theta,j}^\B}\qty(\prod_{k:v_k=1,k\neq j}\qty(1-F(z+\lambda_k))\prod_{k:v_k=0,k\in\B\setminus\qty{i}}F(z+\lambda_k)) \dd z\\
        +\sum_{\theta=1}^{\wm}\int_{\nu-\widetilde{\lambda}_\theta}^\infty f(z+\lambda_i)F(z+\lambda_j)\sum_{\bm{v}\in\mathcal{S}_{i,\theta,j}^\B}\qty(\prod_{k:v_k=1}\qty(1-F(z+\lambda_k))\prod_{k:v_k=0,k\in\B\setminus\qty{i,j}}F(z+\lambda_k)) \dd z.
    \end{multline}
    Taking the limit of $\lambda_j\to\infty$ on both sides of \eqref{eq:to_limit}, since $\lim_{\lambda_j\to\infty}F(z+\lambda_j)=1$, the first term of the RHS of \eqref{eq:to_limit} vanishes, and the second term becomes independent of $\lambda_j$. 
    Then, $\lim_{\lambda_j\to\infty}\phi_i(\lambda;\dis,\wm,\B)$ can be expressed as
    \begin{align*}
        &\lim_{\lambda_j\to\infty}\phi_i(\lambda;\dis,\wm,\B)\\
        =&\sum_{\theta=1}^{\wm}\int_{\nu-\widetilde{\lambda}_\theta}^\infty f(z+\lambda_i)\sum_{\bm{v}\in\mathcal{S}_{i,\theta,j}^\B}\qty(\prod_{k:v_k=1}\qty(1-F(z+\lambda_k))\prod_{k:v_k=0,k\in\B\setminus\qty{i,j}}F(z+\lambda_k)) \dd z\\
        =&\sum_{\theta=1}^{\wm}\int_{\nu-\widetilde{\lambda}_\theta}^\infty f(z+\lambda_i)\sum_{\bm{v}\in\mathcal{S}_{i,\theta}^{\B\setminus\qty{j}}}\qty(\prod_{k:v_k=1}\qty(1-F(z+\lambda_k))\prod_{k:v_k=0,k\in\B\setminus\qty{i,j}}F(z+\lambda_k)) \dd z,\\
        =&\phi_i(\lambda;\dis,\wm,\B\setminus\qty{j}).\numberthis\label{eq:S_prime}
    \end{align*}
    By the same argument, we also have
    \begin{equation}\label{eq:S_prime_J}
        \lim_{\lambda_j\to\infty}J_i(\lambda;\dis,\wm,\B)=J_i(\lambda;\dis,\wm,\B\setminus\qty{j}).
    \end{equation}
    By Lemma~\ref{lem:important_inequality}, \eqref{eq:S_prime} and \eqref{eq:S_prime_J}, we have
    \begin{align*}
        \frac{J_i(\lambda;\dis,\wm,\B)}{\phi_i(\lambda;\dis,\wm,\B)}
        &\leq
        \lim_{\lambda_j\to\infty}\frac{J_i(\lambda;\dis,\wm-1,\B\setminus\qty{j})}{\phi_i(\lambda;\dis,\wm-1,\B\setminus\qty{j})}\lor
        \lim_{\lambda_j\to\infty}\frac{J_i(\lambda;\dis,\wm,\B)}{\phi_i(\lambda;\dis,\wm,\B)}\\
        &=\frac{J_i(\lambda;\dis,\wm-1,\B\setminus\qty{j})}{\phi_i(\lambda;\dis,\wm-1,\B\setminus\qty{j})}\lor
        \frac{\lim_{\lambda_j\to\infty}J_i(\lambda;\dis,\wm,\B)}{\lim_{\lambda_j\to\infty}\phi_i(\lambda;\dis,\wm,\B)}\numberthis\label{eq:ND_limit}\\
        &=
        \frac{J_i(\lambda;\dis,\wm-1,\B\setminus\qty{j})}{\phi_i(\lambda;\dis,\wm-1,\B\setminus\qty{j})}\lor
        \frac{J_i(\lambda;\dis,\wm,\B\setminus\qty{j})}{\phi_i(\lambda;\dis,\wm,\B\setminus\qty{j})}.
    \end{align*}
    where \eqref{eq:ND_limit} holds since both $\lim_{\lambda_j\to\infty}J_i(\lambda;\dis,\wm,\B)$ and $\lim_{\lambda_j\to\infty}\phi_i(\lambda;\dis,\wm,\B)$ exist and nonzero.

    \item \paragraph{Inequality \eqref{eq:important3}}
    This result follows as a special case of \eqref{eq:important2}, where in \eqref{eq:S_prime} we have
    $\mathcal{S}_{i,\wm}^{\B\setminus\qty{j}}=\varnothing$, since there exists no $\bm{v}\in\qty{0,1}^d$ satisfying $\normv_1=\wm-1=\left\lvert\B\right\rvert-1$, $v_i=0$ and $v_k=0$ for all $k\notin\B\setminus\qty{j}$ simultaneously. Therefore, we have
    \begin{equation*}
        \lim_{\lambda_j\to\infty}\phi_i(\lambda;\dis,\wm,\B)=\phi_i(\lambda;\dis,\wm-1,\B\setminus\qty{j})
    \end{equation*}
    and
    \begin{equation*}
        \lim_{\lambda_j\to\infty}J_i(\lambda;\dis,\wm,\B)=J_i(\lambda;\dis,\wm-1,\B\setminus\qty{j}),
    \end{equation*}
    which conclude the proof.
\end{proof}

\subsubsection{Proof of Lemma~\ref{lem:lambda_star}}

\textbf{Lemma~\ref{lem:lambda_star} (Restated)}
\textit{It holds that
\begin{equation*}
    \frac{J_i(\lambda;\dis)}{\phi_i(\lambda;\dis)}\leq
    \max_{\substack{w\in\qty{0}\cup[(m\land i)-1]\\\theta\in[(m\land i)-w]}}\qty{\frac{J_{i,\theta}(\lambda^*;\dis,\B_{i,w})}{\phi_{i,\theta}(\lambda^*;\dis,\B_{i,w})}},
\end{equation*}
where
\begin{equation*}
    \B_{i,0}=[i],
    \B_{i,w}=[i]\setminus[w], \text{ and }
    \lambda_k^* = 
    \begin{cases}
        \lambda_i, & \text{if } k\leq i,\\
        \lambda_k, & \text{if } k> i.
    \end{cases}
\end{equation*}
}    

\begin{proof}
    In this proof, we locally use $\lambda\pn{0},\lambda\pn{1},\lambda\pn{2},\cdots,\lambda\pn{i-1}$ to denote a sequence of $d$-dimensional vectors defined as follows.  
    Define $\lambda\pn{0}=\lambda$, for $j\leq i$ we have
    \begin{equation*}
        \lambda\pn{j}_k = 
        \begin{cases}
            \lambda_i, & \text{if } k=j,\\
            \lambda\pn{j-1}_k, & \text{otherwise,}
        \end{cases}
        \text{ which implies }
        \lambda\pn{j}_k = 
        \begin{cases}
            \lambda_i, & \text{if } k\in[j]\cup\qty{i},\\
            \lambda_k, & \text{otherwise.}
        \end{cases}
    \end{equation*}
    Consequently, we have $\lambda\pn{i-1}=\lambda^*.$

    We are now ready to derive the result. 
    Firstly, to address all $\lambda_j\leq\lambda_i$, 
    for $i=1$, we only need to show that
    \begin{equation}\label{eq:induction_statement}
        \frac{J_i(\lambda;\dis)}{\phi_i(\lambda;\dis)}\leq
        \max_{w\in\qty{0}\cup[k\land(m-1)]}\qty{\frac{J_i(\lambda\pn{k};\dis,m-w,\B_{d,w})}{\phi_i(\lambda\pn{k};\dis,m-w,\B_{d,w})}}.
    \end{equation}
    holds for $k=0$, which is trivial.
    For $i\geq 2$, we prove \eqref{eq:induction_statement} holds for all $k\in[i-1]$ by mathematical induction.
    We have
    \begin{equation*}
        \frac{J_i(\lambda;\dis)}{\phi_i(\lambda;\dis)} = 
        \frac{J_i(\lambda;\dis,m,\B_{d,0})}{\phi_i(\lambda;\dis,m,\B_{d,0})}.
    \end{equation*}
    We begin by verifying the base case of the induction. When $k=1$, the statement becomes
    \begin{equation*}
        \frac{J_i(\lambda;\dis,m,\B_{d,0})}{\phi_i(\lambda;\dis,m,\B_{d,0})}
        \leq
        \max_{w\in\qty{0}\cup[1\land(m-1)]}\qty{\frac{J_i(\lambda\pn{1};\dis,m-w,\B_{d,w})}{\phi_i(\lambda\pn{1};\dis,m-w,\B_{d,w})}}.
    \end{equation*} 
    If $m=1$, the statement is immediate, since $J_i(\lambda;\dis)/\phi_i(\lambda;\dis)$ is expressed as
    \begin{equation*}
        \frac{J_i(\lambda;\dis)}{\phi_i(\lambda;\dis)}
        =\frac{\int_{\nu-\min_{j\in[d]}\lambda_j}^{\infty}\frac{1}{z+\lambda_i}f(z+\lambda_i)\prod_{j\neq i}F(z+\lambda_j)\dd z}{\int_{\nu-\min_{j\in[d]}\lambda_j}^{\infty}f(z+\lambda_i)\prod_{j\neq i}F(z+\lambda_j)\dd z},
    \end{equation*}
    which is monotonically increasing in all $\lambda_j\neq\lambda_i$ by Lemma~\ref{lem:both_increasing}. 
    Otherwise, by applying Lemma~\ref{lem:important_inequality}, we have
    \begin{align*}
        \frac{J_i(\lambda;\dis,m,\B_{d,0})}{\phi_i(\lambda;\dis,m,\B_{d,0})}
        &\leq
        \frac{J_i(\lambda\pn{1};\dis,m-1,\B_{d,1})}{\phi_i(\lambda\pn{1};\dis,m-1,\B_{d,1})}
        \lor
        \frac{J_i(\lambda\pn{1};\dis,m,\B_{d,0})}{\phi_i(\lambda\pn{1};\dis,m,\B_{d,0})}\\
        &=\max_{w\in\qty{0,1}}\qty{\frac{J_i(\lambda\pn{1};\dis,m-w,\B_{d,w})}{\phi_i(\lambda\pn{1};\dis,m-w,\B_{d,w})}}.
    \end{align*}
    Therefore, the statement holds for $k=1$.

    We assume, as the inductive hypothesis, that the statement holds for $k=u< i-1$, i.e.,
    \begin{equation}\label{eq:induction_hypothesis}
        \frac{J_i(\lambda;\dis,m,\B_{d,0})}{\phi_i(\lambda;\dis,m,\B_{d,0})}
        \leq
        \max_{w\in\qty{0}\cup[u\land(m-1)]}\qty{\frac{J_i(\lambda\pn{u};\dis,m-w,\B_{d,w})}{\phi_i(\lambda\pn{u};\dis,m-w,\B_{d,w})}}.
    \end{equation}
    If we can prove the statement holds for $k=u+1$, then by induction, the statement holds for all $k\leq i-1$, thereby establishing the desired result in \eqref{eq:induction_statement}.
    Now we aim to prove it for $k=u+1\leq i-1$, i.e., we want to show that
    \begin{equation*}
        \frac{J_i(\lambda;\dis,m,\B_{d,0})}{\phi_i(\lambda;\dis,m,\B_{d,0})}
        \leq
        \max_{w\in\qty{0}\cup[(u+1)\land(m-1)]}\qty{\frac{J_i(\lambda\pn{u+1};\dis,m-w,\B_{d,w})}{\phi_i(\lambda\pn{u+1};\dis,m-w,\B_{d,w})}}.
    \end{equation*}
    To prove this, it suffices to show the following inequality holds:
    \begin{multline}\label{eq:for_induction_hypothesis}
        \max_{w\in\qty{0}\cup[u\land(m-1)]}\qty{\frac{J_i(\lambda\pn{u};\dis,m-w,\B_{d,w})}{\phi_i(\lambda\pn{u};\dis,m-w,\B_{d,w})}}\leq\\
        \max_{w\in\qty{0}\cup[(u+1)\land(m-1)]}\qty{\frac{J_i(\lambda\pn{u+1};\dis,m-w,\B_{d,w})}{\phi_i(\lambda\pn{u+1};\dis,m-w,\B_{d,w})}},
    \end{multline}
    since this, together with the induction hypothesis \eqref{eq:induction_hypothesis}, implies that
    \begin{align*}
        \frac{J_i(\lambda;\dis,m,\B_{d,0})}{\phi_i(\lambda;\dis,m,\B_{d,0})}
        &\leq
        \max_{w\in\qty{0}\cup[u\land(m-1)]}\qty{\frac{J_i(\lambda\pn{u};\dis,m-w,\B_{d,w})}{\phi_i(\lambda\pn{u};\dis,m-w,\B_{d,w})}}\\
        &\leq
        \max_{w\in\qty{0}\cup[(u+1)\land(m-1)]}\qty{\frac{J_i(\lambda\pn{u+1};\dis,m-w,\B_{d,w})}{\phi_i(\lambda\pn{u+1};\dis,m-w,\B_{d,w})}}.
    \end{align*}
    Now we prove \eqref{eq:for_induction_hypothesis} holds for $u<i-1$. 
    For each term in the LHS of \eqref{eq:for_induction_hypothesis} given by 
    $$J_i(\lambda\pn{u};\dis,m-w_0,\B_{d,w_0})/\phi_i(\lambda\pn{u};\dis,m-w_0,\B_{d,w_0}),\text{ where } w_0\in\qty{0}\cup[u\land(m-1)],$$
    we consider the following two cases.
    \begin{itemize}
        \item \textbf{Case 1:} $w_0=m-1$.\\
        In this case, we have $m-w_0=1$. 
        Similarly to the analysis on the base case, by Lemma~\ref{lem:both_increasing}, for any $j\in\B_{d,w_0}\setminus\qty{i}$,
        $$J_i(\lambda\pn{u};\dis,1,\B_{d,w_0})/\phi_i(\lambda\pn{u};\dis,1,\B_{d,w_0})$$
        is monotonically increasing in $\lambda_j$. Applying this to $\lambda\pn{u}_{u+1}$, we have
        \begin{align*}
            \frac{J_i(\lambda\pn{u};\dis,1,\B_{d,w_0})}{\phi_i(\lambda\pn{u};\dis,1,\B_{d,w_0})}
            &\leq
            \frac{J_i(\lambda\pn{u+1};\dis,1,\B_{d,w_0})}{\phi_i(\lambda\pn{u+1};\dis,1,\B_{d,w_0})}\\
            &\leq
            \max_{w\in\qty{0}\cup[(u+1)\land(m-1)]}\qty{\frac{J_i(\lambda\pn{u+1};\dis,m-w,\B_{d,w})}{\phi_i(\lambda\pn{u+1};\dis,m-w,\B_{d,w})}},\numberthis\label{eq:proof_case1}
        \end{align*}
        where the last inequality holds since $w_0\in\qty{0}\cup[u\land(m-1)]$ and $m-w_0=1$.

        \item \textbf{Case 2:} $w_0\leq m-2$.\\
        In this case, since $w_0\leq m-2$, we have $w_0<w_0+1\leq m-1$. On the other hand, since $w_0\in\qty{0}\cup[u\land(m-1)]$, we have $w_0\leq u$ and thus $w_0<w_0+1\leq u+1$. Combining these, we have $\qty{w_0,w_0+1}\subset\qty{0}\cup[(u+1)\land(m-1)]$.
        Since $m-w_0\geq 2$, by applying Lemma~\ref{lem:important_inequality} again, we have
        \begin{align*}
            &\frac{J_i(\lambda\pn{u};\dis,m-w_0,\B_{d,w_0})}{\phi_i(\lambda\pn{u};\dis,m-w_0,\B_{d,w_0})}\\
            \leq& \frac{J_i(\lambda\pn{u+1};\dis,m-(w_0+1),\B_{d,w_0+1})}{\phi_i(\lambda\pn{u+1};\dis,m-(w_0+1),\B_{d,w_0+1})}
            \lor
            \frac{J_i(\lambda\pn{u+1};\dis,m-w_0,\B_{d,w_0})}{\phi_i(\lambda\pn{u+1};\dis,m-w_0,\B_{d,w_0})}\\
            \leq&
            \max_{w\in\qty{0}\cup[(u+1)\land(m-1)]}\qty{\frac{J_i(\lambda\pn{u+1};\dis,m-w,\B_{d,w})}{\phi_i(\lambda\pn{u+1};\dis,m-w,\B_{d,w})}},\numberthis\label{eq:proof_case2}
        \end{align*}
        where the last inequality holds since $\qty{w_0,w_0+1}\subset\qty{0}\cup[(u+1)\land(m-1)]$.
    \end{itemize}
    Combining \eqref{eq:proof_case1} and \eqref{eq:proof_case2}, we see that \eqref{eq:for_induction_hypothesis} holds and thus the statement holds for $k=u+1\leq i-1$, completing the inductive step.
    By the principle of mathematical induction, the statement \eqref{eq:induction_statement} holds for $k\in[i-1]$. By letting $k=i-1$, we have
    \begin{equation}\label{eq:induction_result1}
        \frac{J_i(\lambda;\dis)}{\phi_i(\lambda;\dis)}\leq
        \max_{w\in\qty{0}\cup[(m\land i)-1]}\qty{\frac{J_i(\lambda^*;\dis,m-w,\B_{d,w})}{\phi_i(\lambda^*;\dis,m-w,\B_{d,w})}}.
    \end{equation}
    Now we address all the base-arms in $\B_{d,w}\setminus\B_{i,w}=[d]\setminus[i]$, i.e., all $\lambda_j$ satisfying $\lambda_j>\lambda_i$.
    We prove
    \begin{equation}\label{eq:induction_statement2}
        \frac{J_i(\lambda;\dis)}{\phi_i(\lambda;\dis)}\leq
        \max_{\substack{w\in\qty{0}\cup[(m\land i)-1]\\\wm\in[(m\land (d-k))-w]}}\qty{\frac{J_i(\lambda^*;\dis,\wm,\B_{d-k,w})}{\phi_i(\lambda^*;\dis,\wm,\B_{d-k,w})}}
    \end{equation}
    by mathematical induction on $k\in[d-i]$. 
    We begin by verifying the base case of the induction. When $k=1$, the statement becomes 
    \begin{equation*}
        \frac{J_i(\lambda;\dis)}{\phi_i(\lambda;\dis)}\leq
        \max_{\substack{w\in\qty{0}\cup[(m\land i)-1]\\\wm\in[(m\land (d-1))-w]}}\qty{\frac{J_i(\lambda^*;\dis,\wm,\B_{d-1,w})}{\phi_i(\lambda^*;\dis,\wm,\B_{d-1,w})}}.
    \end{equation*}
    If $m<d$, we consider each term in the RHS of \eqref{eq:induction_result1} in two cases. 
    For $w=w_0\in\qty{0}\cup[(m\land i)-1]$ such that $m-w_0=1$, it follows that
    \begin{align*}
        \frac{J_i(\lambda^*;\dis,m-w_0,\B_{d,w_0})}{\phi_i(\lambda^*;\dis,m-w_0,\B_{d,w_0})}
        &=
        \frac{J_i(\lambda^*;\dis,1,\B_{d,w_0})}{\phi_i(\lambda^*;\dis,1,\B_{d,w_0})}\\
        &=
        \frac{\int_{\nu-\min_{j\in\B_{d,w_0}}\lambda^*_j}^{\infty}\frac{1}{z+\lambda^*_i}f(z+\lambda^*_i)\prod_{j\in\B_{d,w_0}\setminus\qty{i}}F(z+\lambda^*_j)\dd z}{\int_{\nu-\min_{j\in\B_{d,w_0}}\lambda^*_j}^{\infty}f(z+\lambda^*_i)\prod_{j\in\B_{d,w_0}\setminus\qty{i}}F(z+\lambda^*_j)\dd z},
    \end{align*}
    which is monotonically increasing in $\lambda_j$ for all $j\in\B_{d,w_0}\setminus\qty{i}$ by Lemma~\ref{lem:both_increasing}. Taking the limit $\lambda^*_d\to\infty$, we have $\lim_{\lambda^*_d\to\infty}F(z+\lambda^*_d)=1$ and thus
    \begin{align*}
        \frac{J_i(\lambda^*;\dis,m-w_0,\B_{d,w_0})}{\phi_i(\lambda^*;\dis,m-w_0,\B_{d,w_0})}
        &\leq\frac{\int_{\nu-\min_{j\in\B_{d,w_0}}\lambda^*_j}^{\infty}\frac{1}{z+\lambda^*_i}f(z+\lambda^*_i)\prod_{j\in\B_{d,w_0}\setminus\qty{i,d}}F(z+\lambda^*_j)\dd z}{\int_{\nu-\min_{j\in\B_{d,w_0}}\lambda^*_j}^{\infty}f(z+\lambda^*_i)\prod_{j\in\B_{d,w_0}\setminus\qty{i,d}}F(z+\lambda^*_j)\dd z}\\
        &=\frac{J_i(\lambda^*;\dis,m-w_0,\B_{d-1,w_0})}{\phi_i(\lambda^*;\dis,m-w_0,\B_{d-1,w_0})}.
    \end{align*}
    For $w=w_0\in\qty{0}\cup[(m\land i)-1]$ such that $m-w_0\geq 2$, by applyng Lemma~\ref{lem:infty} to base-arm $d$, we obtain
    \begin{equation*}
        \frac{J_i(\lambda^*;\dis,m-w_0,\B_{d,w_0})}{\phi_i(\lambda^*;\dis,m-w_0,\B_{d,w_0})}\leq
        \frac{J_i(\lambda^*;\dis,m-w_0,\B_{d-1,w_0})}{\phi_i(\lambda^*;\dis,m-w_0,\B_{d-1,w_0})}\lor \frac{J_i(\lambda^*;\dis,m-1-w_0,\B_{d-1,w_0})}{\phi_i(\lambda^*;\dis,m-1-w_0,\B_{d-1,w_0})}.
    \end{equation*}
    Combining the above two cases, by \eqref{eq:induction_result1} we have
    \begin{align*}
        \frac{J_i(\lambda;\dis)}{\phi_i(\lambda;\dis)}
        &\leq
        \max_{w\in\qty{0}\cup[(m\land i)-1]}\qty{\frac{J_i(\lambda^*;\dis,m-w,\B_{d-1,w})}{\phi_i(\lambda^*;\dis,m-w,\B_{d-1,w})}, \frac{J_i(\lambda^*;\dis,\qty(m-1-w)\lor 1,\B_{d-1,w})}{\phi_i(\lambda^*;\dis,\qty(m-1-w)\lor 1,\B_{d-1,w})}}\\
        &\leq 
        \max_{\substack{w\in\qty{0}\cup[(m\land i)-1]\\\wm\in[(m\land (d-1))-w]}}\qty{\frac{J_i(\lambda^*;\dis,\wm,\B_{d-1,w})}{\phi_i(\lambda^*;\dis,\wm,\B_{d-1,w})}},
    \end{align*} 
    where the last inequality holds since $m<d$ and $\qty{m-w,\qty(m-1-w)\lor 1}\subset[(m\land (d-1))-w]$ for any $w\in\qty{0}\cup[(m\land i)-1]$. 
    On the other hand, if $m=d$, by Lemma~\ref{lem:infty} and \eqref{eq:induction_result1}, we obtain
    \begin{align*}
        \frac{J_i(\lambda;\dis)}{\phi_i(\lambda;\dis)}
        &\leq
        \max_{w\in\qty{0}\cup[(m\land i)-1]}\qty{\frac{J_i(\lambda^*;\dis,m-1-w,\B_{d-1,w})}{\phi_i(\lambda^*;\dis,m-1-w,\B_{d-1,w})}}\\
        &\leq 
        \max_{\substack{w\in\qty{0}\cup[(m\land i)-1]\\\wm\in[(m\land (d-1))-w]}}\qty{\frac{J_i(\lambda^*;\dis,\wm,\B_{d-1,w})}{\phi_i(\lambda^*;\dis,\wm,\B_{d-1,w})}},
    \end{align*}
    where the last inequality holds since $m-1-w\in[(m\land (d-1))-w]$ for any $w\in\qty{0}\cup[(m\land i)-1]$.
    Therefore, the statement holds for base case $k=1$.
    
    Now we assume, as the inductive hypothesis, that the statement holds for $k=u\leq d-i-1$, i.e.,
    \begin{equation}\label{eq:induction_hypothesis2}
        \frac{J_i(\lambda;\dis)}{\phi_i(\lambda;\dis)}\leq
        \max_{\substack{w\in\qty{0}\cup[(m\land i)-1]\\\wm\in[(m\land (d-u))-w]}}\qty{\frac{J_i(\lambda^*;\dis,\wm,\B_{d-u,w})}{\phi_i(\lambda^*;\dis,\wm,\B_{d-u,w})}}.
    \end{equation}
    If we can prove the statement holds for $k=u+1$, then by induction, the statement holds for all $k\leq d-i$, thereby establishing the desired result in \eqref{eq:induction_statement2}. Now we aim to prove it for $k=u+1\leq d-i$, i.e., we want to show that
    \begin{equation*}
        \frac{J_i(\lambda;\dis)}{\phi_i(\lambda;\dis)}\leq
        \max_{\substack{w\in\qty{0}\cup[(m\land i)-1]\\\wm\in[(m\land (d-u-1))-w]}}\qty{\frac{J_i(\lambda^*;\dis,\wm,\B_{d-u-1,w})}{\phi_i(\lambda^*;\dis,\wm,\B_{d-u-1,w})}}.
    \end{equation*}
    To prove this, by induction hypothesis \eqref{eq:induction_hypothesis2}, we only need to show that the following inequality holds:
    \begin{equation}\label{eq:for_induction_hypothesis2}
        \max_{\substack{w\in\qty{0}\cup[(m\land i)-1]\\\wm\in[(m\land (d-u))-w]}}\qty{\frac{J_i(\lambda^*;\dis,\wm,\B_{d-u,w})}{\phi_i(\lambda^*;\dis,\wm,\B_{d-u,w})}}
        \leq
        \max_{\substack{w\in\qty{0}\cup[(m\land i)-1]\\\wm\in[(m\land (d-u-1))-w]}}\qty{\frac{J_i(\lambda^*;\dis,\wm,\B_{d-u-1,w})}{\phi_i(\lambda^*;\dis,\wm,\B_{d-u-1,w})}}.
    \end{equation}
    We consider each term in the LHS of the above inequality with $w=w_0\in\qty{0}\cup[(m\land i)-1]$ and $\wm=\wm_0\in[(m\land (d-u))-w_0]$ in the following three cases.
    \begin{itemize}
        \item \textbf{Case 1:} $\wm_0=1$.\\
        In this case, since $u\leq d-i-1$ and $w_0\leq (m\land i)-1$, we have 
        $$(d-u-1)-w_0\geq d-(d-i-1)-1-\qty((m\land i)-1)\geq 1=\wm_0.$$
        On the other hand, since $\wm_0\in[(m\land (d-u))-w_0]$, we have $\wm_0\leq m-w_0$. Combining these, it follows that $\wm_0\in[(m\land (d-u-1))-w_0]$. 
        Similarly to the analysis on the base case, by Lemma~\ref{lem:both_increasing}, for any $j\in\B_{d-u,w_0}\setminus\qty{i}$,
        $$J_i(\lambda^*;\dis,\wm_0,\B_{d-u,w_0})/\phi_i(\lambda^*;\dis,\wm_0,\B_{d-u,w_0})$$
        is monotonically increasing in $\lambda_j$. Taking the limit $\lambda^*_{d-u}\to\infty$, it follows that
        \begin{align*}
            \frac{J_i(\lambda^*;\dis,\wm_0,\B_{d-u,w_0})}{\phi_i(\lambda^*;\dis,\wm_0,\B_{d-u,w_0})}
            &\leq
            \frac{J_i(\lambda^*;\dis,\wm_0,\B_{d-u-1,w_0})}{\phi_i(\lambda^*;\dis,\wm_0,\B_{d-u-1,w_0})}\\
            &\leq
            \max_{\substack{w\in\qty{0}\cup[(m\land i)-1]\\\wm\in[(m\land (d-u-1))-w]}}\qty{\frac{J_i(\lambda^*;\dis,\wm,\B_{d-u-1,w})}{\phi_i(\lambda^*;\dis,\wm,\B_{d-u-1,w})}},
        \end{align*}
        where the last inequality holds since $w_0\in\qty{0}\cup[(m\land i)-1]$ and $\wm_0\in[(m\land (d-u-1))-w_0]$.
        \item \textbf{Case 2:} $\wm_0\geq 2$ and $\wm_0\leq (d-u)-w_0-1$.\\
        In this case, since $\wm_0\in[(m\land (d-u))-w_0]$ we have $\wm_0\leq m-w_0$. Combining it with $\wm_0\leq (d-u)-w_0-1$, it follows that $\wm_0\in[(m\land (d-u-1))-w_0]$. Since $\wm_0\geq 2$, we have $\qty{\wm_0-1,\wm_0}\subset[(m\land (d-u-1))-w_0]$. 
        By applying Lemma~\ref{lem:infty} to base-arm $d-u$, we have
        \begin{align*}
            \frac{J_i(\lambda^*;\dis,\wm_0,\B_{d-u,w_0})}{\phi_i(\lambda^*;\dis,\wm_0,\B_{d-u,w_0})}
            &\leq
            \frac{J_i(\lambda^*;\dis,\wm_0-1,\B_{d-u-1,w_0})}{\phi_i(\lambda^*;\dis,\wm_0-1,\B_{d-u-1,w_0})}
            \lor
            \frac{J_i(\lambda^*;\dis,\wm_0,\B_{d-u-1,w_0})}{\phi_i(\lambda^*;\dis,\wm_0,\B_{d-u-1,w_0})}\\
            &\leq
            \max_{\substack{w\in\qty{0}\cup[(m\land i)-1]\\\wm\in[(m\land (d-u-1))-w]}}\qty{\frac{J_i(\lambda^*;\dis,\wm,\B_{d-u-1,w})}{\phi_i(\lambda^*;\dis,\wm,\B_{d-u-1,w})}},
        \end{align*}
        where the last inequality holds since $w_0\in\qty{0}\cup[(m\land i)-1]$ and $\qty{\wm_0-1,\wm_0}\subset[(m\land (d-u-1))-w_0]$.
        \item \textbf{Case 3:} $\wm_0\geq 2$ and $\wm_0=(d-u)-w_0$.\\
        In this case, since $\wm_0\geq 2$ and $\wm_0=(d-u)-w_0$, we have $\wm_0-1=(d-u-1)-w_0\geq 1$. On the other hand, since $\wm_0\in[(m\land (d-u))-w_0]$ we have $\wm_0-1\leq m-w_0$. Combining these, it follows that $\wm_0-1\in[(m\land (d-u-1))-w_0]$.
        By applying Lemma~\ref{lem:infty} to base-arm $d-u$, we have
        \begin{align*}
            \frac{J_i(\lambda^*;\dis,\wm_0,\B_{d-u,w_0})}{\phi_i(\lambda^*;\dis,\wm_0,\B_{d-u,w_0})}
            &\leq
            \frac{J_i(\lambda^*;\dis,\wm_0-1,\B_{d-u-1,w_0})}{\phi_i(\lambda^*;\dis,\wm_0-1,\B_{d-u-1,w_0})}\\
            &\leq
            \max_{\substack{w\in\qty{0}\cup[(m\land i)-1]\\\wm\in[(m\land (d-u-1))-w]}}\qty{\frac{J_i(\lambda^*;\dis,\wm,\B_{d-u-1,w})}{\phi_i(\lambda^*;\dis,\wm,\B_{d-u-1,w})}},
        \end{align*}
        where the last inequality holds since $w_0\in\qty{0}\cup[(m\land i)-1]$ and $\wm_0-1\in[(m\land (d-u-1))-w_0]$.
    \end{itemize}
    Combining the three cases, we see that \eqref{eq:for_induction_hypothesis2} holds, and thus the statement holds for $k=u+1\leq d-i$, completing the inductive step.
    By induction, the statement \eqref{eq:induction_statement2} holds for $k\in[d-i]$. By letting $k=d-i$, we immediately obtain
    \begin{equation}\label{eq:proof_s}
        \frac{J_i(\lambda;\dis)}{\phi_i(\lambda;\dis)}\leq
        \max_{\substack{w\in\qty{0}\cup[(m\land i)-1]\\\wm\in[(m\land i)-w]}}\qty{\frac{J_i(\lambda^*;\dis,\wm,\B_{i,w})}{\phi_i(\lambda^*;\dis,\wm,\B_{i,w})}}.
    \end{equation}
    Note that we have
    \begin{equation}\label{eq:phi_decomposition}
        \phi_i(\lambda^*;\dis,\wm_0,\B_{i,w_0})=\sum_{\theta=1}^{\wm_0}\phi_{i,\theta}(\lambda^*;\dis,\B_{i,w_0})
    \end{equation}
    and
    \begin{equation}\label{eq:J_decomposition}
        J_i(\lambda^*;\dis,\wm_0,\B_{i,w_0})=\sum_{\theta=1}^{\wm_0}J_{i,\theta}(\lambda^*;\dis,\B_{i,w_0}).
    \end{equation}
    Since each term in the RHS of both \eqref{eq:phi_decomposition} and \eqref{eq:J_decomposition} is positive, we have
    \begin{align*}
        \frac{J_i(\lambda^*;\dis,\wm_0,\B_{i,w_0})}{\phi_i(\lambda^*;\dis,\wm_0,\B_{i,w_0})}&=
        \frac{\sum_{\theta=1}^{\wm_0}J_{i,\theta}(\lambda^*;\dis,\B_{i,w_0})}{\sum_{\theta=1}^{\wm_0}\phi_{i,\theta}(\lambda^*;\dis,\B_{i,w_0})}\\
        &\leq\max_{\theta\in[\wm_0]}\qty{\frac{J_{i,\theta}(\lambda^*;\dis,\B_{i,w_0})}{\phi_{i,\theta}(\lambda^*;\dis,\B_{i,w_0})}}.\numberthis\label{eq:max_theta}
    \end{align*}
    Combining \eqref{eq:proof_s} and \eqref{eq:max_theta}, we have
    \begin{align*}
        \frac{J_i(\lambda;\dis)}{\phi_i(\lambda;\dis)}&\leq
        \max_{\substack{w\in\qty{0}\cup[(m\land i)-1]\\\wm\in[(m\land i)-w]}}\qty{\max_{\theta\in[\wm]}\qty{\frac{J_{i,\theta}(\lambda^*;\dis,\B_{i,w})}{\phi_{i,\theta}(\lambda^*;\dis,\B_{i,w})}}}\\
        &=\max_{\substack{w\in\qty{0}\cup[(m\land i)-1]\\\theta\in[(m\land i)-w]}}\qty{\frac{J_{i,\theta}(\lambda^*;\dis,\B_{i,w})}{\phi_{i,\theta}(\lambda^*;\dis,\B_{i,w})}},
    \end{align*}
    which concludes the proof.
\end{proof}

\subsection{Proof of Lemma~\ref{lem:sigma_i}}

\subsubsection{Pareto Distribution}
Let us consider the case $\dis=\Pdis$. Recall that the probability density function and cumulative distribution function of Pareto distribution are given by
\begin{equation*}
    f(x) = \frac{\alpha}{x^{\alpha+1}},\quad F(x) = 1-x^{-\alpha},\quad x\geq 1.
\end{equation*}
Then, for $w\in\qty{0}\cup[(m\land i)-1],\theta\in[(m\land i)-w]$, we have
\begin{align*}
    J_{i,\theta}(\lambda^*;\Pdis,\B_{i,w})&=
    \binom{i-w-1}{\theta-1}\int_{1-\lambda_\theta}^\infty \frac{f(z+\lambda_i)}{z+\lambda_i}\qty(1-F(z+\lambda_i))^{\theta-1}F^{i-w-\theta}(z+\lambda_i)\dd z\\
    &=\binom{i-w-1}{\theta-1}\int_{1}^\infty \frac{f(z)}{z}\qty(1-F(z))^{\theta-1}F^{i-w-\theta}(z)\dd z
\end{align*}
and
\begin{align*}
    \phi_{i,\theta}(\lambda^*;\Pdis,\B_{i,w})&=
    \binom{i-w-1}{\theta-1}\int_{1-\lambda_\theta}^\infty f(z+\lambda_i)\qty(1-F(z+\lambda_i))^{\theta-1}F^{i-w-\theta}(z+\lambda_i)\dd z\\
    &=\binom{i-w-1}{\theta-1}\int_{1}^\infty f(z)\qty(1-F(z))^{\theta-1}F^{i-w-\theta}(z)\dd z.
\end{align*}
Then, it holds that
\begin{align*}
    J_{i,\theta}(\lambda^*;\Pdis,\B_{i,w}) 
    &=\binom{i-w-1}{\theta-1}\int_1^\infty \frac{\alpha}{z^{\alpha\theta+2}}\qty(1-z^{-\alpha})^{i-w-\theta} \dd z\\
    &=\binom{i-w-1}{\theta-1}\int_0^1 w^{\theta-1+\frac{1}{\alpha}}(1-w)^{i-w-\theta}\dd w\\
    &=\binom{i-w-1}{\theta-1}B\qty(\theta+\frac{1}{\alpha},i+1-w-\theta).
\end{align*}
where $B\qty(a,b)=\int_0^1 t^{a-1}(1-t)^{b-1}\dd t$ denotes the Beta function. Similarly, we have
\begin{align*}
    \phi_{i,\theta}(\lambda^*;\Pdis,\B_{i,w}) &= \binom{i-w-1}{\theta-1}\int_1^\infty \frac{\alpha}{z^{\alpha\theta+1}}\qty(1-z^{-\alpha})^{i-w-\theta} \dd z\\
    &=\binom{i-w-1}{\theta-1}\int_0^1 w^{\theta-1}(1-w)^{i-w-\theta}\dd w\\
    &=\binom{i-w-1}{\theta-1}B\qty(\theta,i+1-w-\theta).
\end{align*}
Therefore, we have
\begin{equation}\label{eq:B_RHS}
    \frac{J_{i,\theta}(\lambda^*;\Pdis,\B_{i,w})}{\phi_{i,\theta}(\lambda^*;\Pdis,\B_{i,w})} = \frac{B\qty(\theta+\frac{1}{\alpha},i+1-w-\theta)}{B\qty(\theta,i+1-w-\theta)}.
\end{equation}
Similarly to the proof of \citet{pmlr-v247-lee24a}, we bound \eqref{eq:B_RHS} as follows. For $\alpha>1$, we have
\begin{align*}
    &\frac{B(\theta + \frac{1}{\alpha}, i+1-w-\theta)}{B(\theta, i+1-w-\theta)}\\
    =& \frac{\Gamma(\theta + \frac{1}{\alpha}) \Gamma(i+1-w-\theta)}{\Gamma(i+1+\frac{1}{\alpha}-w)}
        \frac{\Gamma(i+1-w)}{\Gamma(\theta)\Gamma(i+1-w-\theta)}
        \tag*{(by $B(a,b) = \frac{\Gamma(a)\Gamma(b)}{\Gamma(a+b)}$)} \\
    =& \frac{\Gamma(\theta + \frac{1}{\alpha}) }{\Gamma(i+1+\frac{1}{\alpha}-w)}
    \frac{\Gamma(i+1-w)}{\Gamma(\theta)}\\
    =& \frac{1}{i+\frac{1}{\alpha}-w}
        \frac{\Gamma(\theta + \frac{1}{\alpha})}{\Gamma(\theta)}
        \frac{\Gamma(i+1-w)}{\Gamma(i+\frac{1}{\alpha}-w)}
        \tag*{(by $\Gamma(n) = (n - 1)\Gamma(n - 1)$)} \\
    \leq& \frac{1}{i+\frac{1}{\alpha}-w}
        \left( \theta + \frac{1}{\alpha} \right)^{\frac{1}{\alpha}}
        \left( i+1-w \right)^{1 - \frac{1}{\alpha}}
        \tag*{(by Gautschi's inequality)} \\
    =& \frac{i+1-w}{i+\frac{1}{\alpha}-w}
        \left( \frac{\theta + \frac{1}{\alpha}}{i+1-w} \right)^{\frac{1}{\alpha}} \\
    \leq& \frac{2\alpha}{\alpha + 1}
    \left( \frac{\theta + \frac{1}{\alpha}}{i+1-w} \right)^{\frac{1}{\alpha}}. \tag*{(equality holds when $w=i-1$)}
\end{align*}
Since $w\in\qty{0}\cup[(m\land i)-1],\theta\in[(m-w)\land(i-w)]$, we have 
\begin{equation*}
    \max_{\substack{w\in\qty{0}\cup[(m\land i)-1]\\\theta\in[(m-w)\land(i-w)]}}\frac{\theta + \frac{1}{\alpha}}{i+1-w}
    = \max_{w\in\qty{0}\cup[(m\land i)-1]}\frac{(m\land i)+ \frac{1}{\alpha}-w}{i+1-w}
    =\frac{(m\land i) + \frac{1}{\alpha}}{i+1}.
\end{equation*}
Therefore, we have
\begin{equation*}
    \max_{\substack{w\in\qty{0}\cup[(m\land i)-1]\\\theta\in[(m-w)\land(i-w)]}}\frac{B(\theta + \frac{1}{\alpha}, i+1-w-\theta)}{B(\theta, i+1-w-\theta)}
    \leq
    \frac{2\alpha}{\alpha+1}\qty(\frac{(i\land m)+\frac{1}{\alpha}}i)^{\frac{1}{\alpha}}.
\end{equation*}
Recall that $\sigma_i=i$ as previously noted for notational simplicity. By Lemma~\ref{lem:lambda_star}, for any $i\in[d]$, we have
\begin{equation*}
    \frac{J_i(\lambda;\Pdis)}{\phi_i(\lambda;\Pdis)}\leq
    \max_{\substack{w\in\qty{0}\cup[(m\land i)-1]\\\theta\in[(m-w)\land(i-w)]}}\frac{J_{i,\theta}(\lambda^*;\Pdis,\B_{i,w})}{\phi_{i,\theta}(\lambda^*;\Pdis,\B_{i,w})}\leq \frac{2\alpha}{\alpha+1}\qty(\frac{(\sigma_i\land m)+\frac{1}{\alpha}}{\sigma_i})^{\frac{1}{\alpha}}.
\end{equation*}

\subsubsection{Fr\'{e}chet Distribution}

Before proving the statement in the case $\dis=\Fdis$, we need to give the following lemma.
\begin{lemma}\label{lem:frechet_simplify}
    Let $F(x)$ denote the cumulative distribution function of Fr\'{e}chet distribution with shape $\alpha$. For $\lambda_i\geq 0$, we have
    \begin{equation}\label{eq:frechet_to_prove}
        \frac{\int_0^\infty \frac{f(z+\lambda_i)}{z+\lambda_i}\qty(1-F(z+\lambda_i))^p F^q(z+\lambda_i) \dd z}
        {\int_0^\infty f(z+\lambda_i)\qty(1-F(z+\lambda_i))^p F^q(z+\lambda_i) \dd z}
        \leq
        \frac{\int_0^\infty \frac{f(z+\lambda_i)}{(z+\lambda_i)^{\alpha p+1}}F^q(z+\lambda_i) \dd z}
        {\int_0^\infty \frac{f(z+\lambda_i)}{(z+\lambda_i)^{\alpha p}}F^q(z+\lambda_i) \dd z}.
    \end{equation}
\end{lemma}

\begin{proof}
    By letting $h(x)=f(z+\lambda_i)F^q(z+\lambda_i)$, the inequality \eqref{eq:frechet_to_prove} can be rewritten as
    \begin{equation*}
        \frac{\int_0^\infty \frac{1}{z+\lambda_i}\qty(1-F(z+\lambda_i))^p h(z) \dd z}
        {\int_0^\infty \qty(1-F(z+\lambda_i))^p h(z) \dd z}
        \leq
        \frac{\int_0^\infty \frac{1}{(z+\lambda_i)^{\alpha p+1}}h(z) \dd z}
        {\int_0^\infty \frac{1}{(z+\lambda_i)^{\alpha p}}h(z) \dd z}.
    \end{equation*}
    Equivalently, multiplying both sides by
    $$\int_0^\infty \qty(1-F(z+\lambda_i))^p h(z) \dd z\cdot\int_0^\infty \frac{1}{(z+\lambda_i)^{\alpha p}}h(z) \dd z,$$
    we arrive at
    \begin{multline}\label{eq:frechet_to_prove2}
        \int_0^\infty \frac{h(z)}{(z+\lambda_i)}\qty(1-F(z+\lambda_i))^p \dd z 
        \int_0^\infty \frac{h(z)}{(z+\lambda_i)^{\alpha p}} \dd z
        \leq\\
        \int_0^\infty \frac{h(z)}{(z+\lambda_i)^{\alpha p+1}} \dd z
        \int_0^\infty h(z)\qty(1-F(z+\lambda_i))^p \dd z,
    \end{multline}
    and thus we only need to prove it to conclude the proof. The LHS of \eqref{eq:frechet_to_prove2} can be expressed as
    \begin{align*}
        &\int_0^\infty \frac{h(z)}{(z+\lambda_i)}\qty(1-F(z+\lambda_i))^p \dd z 
        \int_0^\infty \frac{h(z)}{(z+\lambda_i)^{\alpha p}} \dd z\\
        =&\iint_{z,w \geq 0} \frac{h(z)h(w)\qty(1-\exp\qty(-1/(z+\lambda_i)^\alpha))^p}{(z+\lambda_i)(w+\lambda_i)^{\alpha p}} \dd z \dd w\\
        =&\frac{1}{2}\iint_{z,w \geq 0} h(z)h(w)\qty(\frac{\qty(1-\exp\qty(-1/(z+\lambda_i)^\alpha))^p}{(z+\lambda_i)(w+\lambda_i)^{\alpha p}}+\frac{\qty(1-\exp\qty(-1/(w+\lambda_i)^\alpha))^p}{(w+\lambda_i)(z+\lambda_i)^{\alpha p}}) \dd z \dd w
    \end{align*}
    and the RHS of \eqref{eq:frechet_to_prove2} can be expressed as
    \begin{align*}
        &\int_0^\infty \frac{h(z)}{(z+\lambda_i)^{\alpha p+1}} \dd z
        \int_0^\infty h(z)\qty(1-F(z+\lambda_i))^p \dd z\\
        =&\iint_{z,w \geq 0} \frac{h(z)h(w)\qty(1-\exp\qty(-1/(w+\lambda_i)^\alpha))^n}{(z+\lambda_i)^{\alpha n+1}} \dd z \dd w\\
        =&\frac{1}{2}\iint_{z,w \geq 0} h(z)h(w)\qty(\frac{\qty(1-\exp\qty(-1/(w+\lambda_i)^\alpha))^n}{(z+\lambda_i)^{\alpha n+1}}+\frac{\qty(1-\exp\qty(-1/(z+\lambda_i)^\alpha))^n}{(w+\lambda_i)^{\alpha n+1}}) \dd z \dd w.
    \end{align*}
    By an elementary calculation we can see
    \begin{align*}
        &\frac{\qty(1-\exp\qty(-1/(z+\lambda_i)^\alpha))^p}{(z+\lambda_i)(w+\lambda_i)^{\alpha p}}+\frac{\qty(1-\exp\qty(-1/(w+\lambda_i)^\alpha))^p}{(w+\lambda_i)(z+\lambda_i)^{\alpha p}}\\
        -&\frac{\qty(1-\exp\qty(-1/(w+\lambda_i)^\alpha))^p}{(z+\lambda_i)^{\alpha p+1}}-\frac{\qty(1-\exp\qty(-1/(z+\lambda_i)^\alpha))^p}{(w+\lambda_i)^{\alpha p+1}}\\
        =& \frac{(z+\lambda_i)^{\alpha p}\qty(1-\exp\qty(-1/(z+\lambda_i)^\alpha))^p - (w+\lambda_i)^{\alpha p}\qty(1-\exp\qty(-1/(w+\lambda_i)^\alpha))^p}{(z+\lambda_i)^{\alpha p+1}(w+\lambda_i)^{\alpha p}} \\
        -&\frac{(w+\lambda_i)^{\alpha p}\qty(1-\exp\qty(-1/(w+\lambda_i)^\alpha))^p - (z+\lambda_i)^{\alpha p}\qty(1-\exp\qty(-1/(z+\lambda_i)^\alpha))^p}{(w+\lambda_i)^{\alpha p+1}(z+\lambda_i)^{\alpha p}} \\
        =&\frac{w-z}{(z+\lambda_i)^{\alpha p+1}(w+\lambda_i)^{\alpha p+1}}\qty((z+\lambda_i)^{\alpha p}\qty(1-\exp\qty(-1/(z+\lambda_i)^\alpha))^p - (w+\lambda_i)^{\alpha p}\qty(1-\exp\qty(-1/(w+\lambda_i)^\alpha))^p) \numberthis\label{eq:frechet_zw_mono}.
    \end{align*}
    By letting $t(x)=(x+\lambda_i)^\alpha$, we have
    \begin{equation}\label{eq:change_variable_mono}
        (x+\lambda_i)^{\alpha n}\qty(1-\exp\qty(-1/(x+\lambda_i)^\alpha))^n=\qty(t\qty(1-e^{-1/t}))^n.
    \end{equation}
    Since $t(x)$ is monotonically increasing in $x$, and $t\qty(1-e^{-1/t})>0$, the LHS of \eqref{eq:change_variable_mono} is monotonic in the same direction as $t\qty(1-e^{-1/t})$, whose derivative is expressed as
    \begin{equation*}
        1-e^{-1/t}-\frac{1}{t}e^{-1/t} \geq 1-e^{-1/t}-(e^{1/t}-1)e^{-1/t}=0,
    \end{equation*}
    where the inequality holds since $\frac{1}{t}\leq e^{1/t}-1$. Therefore, the LHS of \eqref{eq:change_variable_mono} is monotonically increasing in $x$. This implies that the expression in \eqref{eq:frechet_zw_mono} is non-positivive, which concludes the proof.
\end{proof}

We now prove Lemma~\ref{lem:sigma_i} in the case $\dis=\Fdis$ by applying Lemma~\ref{lem:frechet_simplify}. Recall that the probability density function and cumulative distribution function of Pareto distribution are given by
\begin{equation*}
    f(x) = \frac{\alpha}{x^{\alpha+1}}e^{-1/x^\alpha},\quad F(x) = e^{-1/x^\alpha},\quad x\geq 0.
\end{equation*}
Then, for $w\in\qty{0}\cup[(m\land i)-1],\theta\in[(m\land i)-w]$, we have
\begin{align*}
    J_{i,\theta}(\lambda^*;\Fdis,\B_{i,w})&=
    \binom{i-w-1}{\theta-1}\int_{-\lambda_\theta}^\infty \frac{f(z+\lambda_i)}{z+\lambda_i}\qty(1-F(z+\lambda_i))^{\theta-1}F^{i-w-\theta}(z+\lambda_i)\dd z\\
    &=\binom{i-w-1}{\theta-1}\int_0^\infty \frac{f(z)}{z}\qty(1-F(z))^{\theta-1}F^{i-w-\theta}(z)\dd z
\end{align*}
and
\begin{align*}
    \phi_{i,\theta}(\lambda^*;\Fdis,\B_{i,w})&=
    \binom{i-w-1}{\theta-1}\int_{-\lambda_\theta}^\infty f(z+\lambda_i)\qty(1-F(z+\lambda_i))^{\theta-1}F^{i-w-\theta}(z+\lambda_i)\dd z\\
    &=\binom{i-w-1}{\theta-1}\int_0^\infty f(z)\qty(1-F(z))^{\theta-1}F^{i-w-\theta}(z)\dd z.
\end{align*}
Define
\begin{equation*}
    I_{i,n}\qty(q;\Fdis)=\int_0^\infty \frac{1}{z^n}e^{-q/z^\alpha}.
\end{equation*}
Then, by Lemma~\ref{lem:frechet_simplify}, we immediately obtain
\begin{equation}\label{eq:frechet_I}
    \frac{J_{i,\theta}(\lambda^*;\Fdis,\B_{i,w})}{\phi_{i,\theta}(\lambda^*;\Fdis,\B_{i,w})}
    \leq
    \frac{I_{i,\alpha\theta+2}\qty(i+1-w-\theta;\Fdis)}{I_{i,\alpha\theta+1}\qty(i+1-w-\theta;\Fdis)}.
\end{equation}
Similarly to the proofs of \citet{pmlr-v201-honda23a} and \citet{pmlr-v247-lee24a}, we bound the RHS of \eqref{eq:frechet_I} as follows. By letting $u=q/z^\alpha$, both $I_{i,\alpha\theta+2}(q;\alpha)$ and $I_{i,\alpha\theta+1}(q;\alpha)$ can be expressed by Gamma function $\Gamma(k)=\int_0^\infty e^{-t}t^{k-1} \dd t$ as
\begin{align*}
    I_{i,\alpha\theta+2}\qty(q;\Fdis)&=\int_0^\infty \frac{1}{z^{\alpha\theta+2}}\exp\qty(-q/z^\alpha) \dd z \\
    &=\frac{1}{\alpha q}\int_0^\infty \qty(\frac{u}{q})^{\theta+\frac{1}{\alpha}-1}e^{-u} \dd u\\
    &=\frac{1}{\alpha q^{\theta+\frac{1}{\alpha}}}\int_0^\infty u^{\theta+\frac{1}{\alpha}-1}e^{-u} \dd u\\
    &=\frac{1}{\alpha q^{\theta+\frac{1}{\alpha}}}\Gamma\qty(\theta+\frac{1}{\alpha}),
\end{align*}
and
\begin{align*}
    I_{i,\alpha\theta+1}\qty(q;\Fdis)&=\int_0^\infty \frac{1}{z^{\alpha\theta+1}}\exp\qty(-q/z^\alpha) \dd z \\
    &=\frac{1}{\alpha q}\int_0^\infty \qty(\frac{u}{q})^{\theta-1}e^{-u} \dd u\\
    &=\frac{1}{\alpha q^{\theta}}\int_0^\infty u^{\theta-1}e^{-u} \dd u\\
    &=\frac{1}{\alpha q^{\theta}}\Gamma\qty(\theta).
\end{align*}
Replacing $q$ with $i+1-w-\theta$, we have
\begin{equation}\label{eq:frechet_ratio}
    \frac{I_{i,\alpha\theta+2}\qty(i+1-w-\theta;\Fdis)}{I_{i,\alpha\theta+1}\qty(i+1-w-\theta;\Fdis)}=\frac{1}{\sqrt[\alpha]{i+1-w-\theta} }\frac{\Gamma\qty(\theta+\frac{1}{\alpha})}{\Gamma\qty(\theta)}.
\end{equation}
By Lemma~\ref{lem:Gautschi}, Gautschi's inequality, we have
\begin{equation*}
    \frac{\Gamma\qty(\theta+\frac{1}{\alpha})}{\Gamma\qty(\theta)}
    \leq \qty(\theta+\frac{1}{\alpha})^{\frac{1}{\alpha}}.
\end{equation*}
Combining this result with \eqref{eq:frechet_ratio}, we obtain
\begin{equation*}
    \frac{I_{i,\alpha\theta+2}\qty(i+1-w-\theta;\Fdis)}{I_{i,\alpha\theta+1}\qty(i+1-w-\theta;\Fdis)}
    \leq\qty(\frac{\theta+\frac{1}{\alpha}}{i+1-w-\theta})^{\frac{1}{\alpha}}.
\end{equation*}
Since $w\in\qty{0}\cup[(m\land i)-1],\theta\in[(m-w)\land(i-w)]$, we have
\begin{align*}
    \max_{\substack{w\in\qty{0}\cup[(m\land i)-1]\\\theta\in[(m-w)\land(i-w)]}}\frac{\theta+\frac{1}{\alpha}}{i+1-w-\theta}
    &=\max_{w\in\qty{0}\cup[(m\land i)-1]}\frac{(m\land i)+\frac{1}{\alpha}-w}{i+1-(m\land i)}\\
    &=\frac{(m\land i)+\frac{1}{\alpha}}{i+1-(m\land i)}\\
    &=\frac{(m\land i)+\frac{1}{\alpha}}{(i-m+1)\lor 1}.
\end{align*}
Recall that $\sigma_i=i$ as previously noted for notational simplicity. By Lemma~\ref{lem:lambda_star} and \eqref{eq:frechet_I}, for any $i\in[d]$, we have
\begin{align*}
    \frac{J_i(\lambda;\Fdis)}{\phi_i(\lambda;\Fdis)}&\leq
    \max_{\substack{w\in\qty{0}\cup[(m\land i)-1]\\\theta\in[(m-w)\land(i-w)]}}\frac{J_{i,\theta}(\lambda^*;\Pdis,\B_{i,w})}{\phi_{i,\theta}(\lambda^*;\Pdis,\B_{i,w})}\\
    &\leq\max_{\substack{w\in\qty{0}\cup[(m\land i)-1]\\\theta\in[(m-w)\land(i-w)]}}\frac{I_{i,\alpha\theta+2}(\lambda^*,i+1-w-\theta;\Fdis)}{I_{i,\alpha\theta+1}(\lambda^*,i+1-w-\theta;\Fdis)}\\
    &\leq \qty(\frac{(m\land \sigma_i)+\frac{1}{\alpha}}{(\sigma_i-m+1)\lor 1})^{\frac{1}{\alpha}}.
\end{align*}

\subsection{Proof of Lemma~\ref{lem:each_term}}

Following the proofs of \citet{pmlr-v201-honda23a} and \citet{pmlr-v247-lee24a}, we extend the statement to combinatorial semi-bandit setting.

\begin{proof}
    Define
    \begin{equation*}
        \underline{\Omega}=\qty{r:\qty[\argmin\limits_{a \in \nec} \left\{a^\top (\eta (\hat{L}_t + (\ell_{t,i}\widehat{w_{t,i}^{-1}}) e_i) - r)\right\}]_i=1}
    \end{equation*}
    and
    \begin{equation*}
        \overline{\Omega } =\qty{r:\qty[\argmin\limits_{a \in \nec} \left\{a^\top (\eta (\hat{L}_t + \hat{\ell}_t) - r)\right\}]_i=1}.
    \end{equation*}
    Then, we have
    \begin{equation*}
        \phi_i\qty(\eta \qty(\hat{L}_t + \qty(\ell_{t,i}\widehat{w_{t,i}^{-1}}) e_i);\dis)=\sP_{r\sim\dis}\qty(\underline{\Omega}),\quad \phi_i\qty(\eta \qty(\hat{L}_t+\hat{\ell}_t);\dis)=\sP_{r\sim\dis}\qty(\overline{\Omega}).
    \end{equation*}
    Since $\underline{\Omega}\subset\overline{\Omega}$, we immediately have
    \begin{equation*}
        \phi_i\qty(\eta \qty(\hat{L}_t + \qty(\ell_{t,i}\widehat{w_{t,i}^{-1}}) e_i);\dis)\leq\phi_i\qty(\eta \qty(\hat{L}_t+\hat{\ell}_t);\dis),
    \end{equation*}
    with which we have
    \begin{align*}
        \phi_i\qty(\eta \hat{L}_t;\dis)-\phi_i\qty(\eta \qty(\hat{L}_t+\hat{\ell}_t);\dis)
        &\leq\phi_i\qty(\eta \hat{L}_t;\dis) - \phi_i\qty(\eta \qty(\hat{L}_t + \qty(\ell_{t,i}\widehat{w_{t,i}^{-1}}) e_i);\dis)\\ 
        &= \int_0^{\eta \ell_{t,i} \widehat{w_{t,i}^{-1}}} -\phi_i'\qty(\eta \hat{L}_t + x e_i;\dis) \dd x.\numberthis\label{eq:phi_i_diff}
    \end{align*}
    Recalling that $\phi_i(\lambda;\dis)$ is expressed as
    \begin{equation*}
        \phi_i(\lambda;\dis)=
        \sum_{\theta=1}^{m}\int_{\nu-\widetilde{\lambda}_\theta}^\infty f(z+\lambda_i)\sum_{\bm{v}\in\mathcal{S}_{i,\theta}}\qty(\prod_{j:v_j=1}\qty(1-F(z+\lambda_j))\prod_{j:v_j=0,j\neq i}F(z+\lambda_j)) \dd z,
    \end{equation*}
    we see that $\phi'_i(\lambda;\dis)=\frac{\partial \phi_i}{\partial \lambda_i}(\lambda;\dis)$ is expressed as
    \begin{equation*}
        \phi'_i(\lambda;\dis)=\\
        \sum_{\theta=1}^{m}\int_{\nu-\widetilde{\lambda}_\theta}^\infty f'(z+\lambda_i)\sum_{\bm{v}\in\mathcal{S}_{i,\theta}}\qty(\prod_{j:v_j=1}\qty(1-F(z+\lambda_j))\prod_{j:v_j=0,j\neq i}F(z+\lambda_j)) \dd z.
    \end{equation*}
    Now we divide the proof into two cases.
    \item \paragraph{Fr\'{e}chet distribution}
    When $\dis=\Fdis$, since for $x>0$,
    \begin{equation*}
        f'(x)=\qty(-\frac{\alpha(\alpha+1)}{x^{\alpha+2}}+\frac{\alpha^2}{x^{2(\alpha+1)}})e^{-1/x^\alpha},
    \end{equation*}
    we have
    \begin{equation*}
        -f'(x)=\qty(\frac{\alpha(\alpha+1)}{x^{\alpha+2}}-\frac{\alpha^2}{x^{2(\alpha+1)}})e^{-1/x^\alpha}\leq
        \frac{\alpha(\alpha+1)}{x^{\alpha+2}}e^{-1/x^\alpha}
        =\frac{\alpha+1}{x}f(x).
    \end{equation*}
    Therefore, by \eqref{eq:phi_i_diff} we have
    \begin{align*}
        \phi_i\qty(\eta \hat{L}_t;\Fdis)-\phi_i\qty(\eta \qty(\hat{L}_t+\hat{\ell}_t);\Fdis)&\leq
        (\alpha+1)\int_0^{\eta \ell_{t,i} \widehat{w_{t,i}^{-1}}} J_i\qty(\eta \hat{L}_t + x e_i;\Fdis) \dd x \\
        &\leq (\alpha+1)\int_0^{\eta \ell_{t,i} \widehat{w_{t,i}^{-1}}} J_i\qty(\eta \hat{L}_t;\Fdis) \dd x\numberthis\label{eq:J_mono_frechet}\\
        &=(\alpha+1)\eta \ell_{t,i}  J_i\qty(\eta \hat{L}_t;\Fdis)\widehat{w_{t,i}^{-1}},
    \end{align*}
    where \eqref{eq:J_mono_frechet} follows from the monotonicity of $J_i\qty(\eta \hat{L}_t;\Fdis)$.
    \item \paragraph{Pareto distribution}
    When $\dis=\Pdis$, since for $x>1$,
    \begin{equation*}
        f'(x)=-\alpha(\alpha+1)x^{-(\alpha+2)},
    \end{equation*}
    by \eqref{eq:phi_i_diff} we have
    \begin{align*}
        \phi_i\qty(\eta \hat{L}_t;\Pdis)-\phi_i\qty(\eta \qty(\hat{L}_t+\hat{\ell}_t);\Pdis)&=
        (\alpha+1)\int_0^{\eta \ell_{t,i} \widehat{w_{t,i}^{-1}}} J_i\qty(\eta \hat{L}_t + x e_i;\Pdis) \dd x \\
        &\leq (\alpha+1)\int_0^{\eta \ell_{t,i} \widehat{w_{t,i}^{-1}}} J_i\qty(\eta \hat{L}_t;\Pdis) \dd x\numberthis\label{eq:J_mono_pareto}\\
        &=(\alpha+1)\eta \ell_{t,i}  J_i\qty(\eta \hat{L}_t;\Pdis)\widehat{w_{t,i}^{-1}},
    \end{align*}
    where \eqref{eq:J_mono_pareto} follows from the monotonicity of $J_i\qty(\eta \hat{L}_t;\Pdis)$.

    Here note that $\widehat{{w_{t,i}^{-1}}}$ follows the geometric distribution with expectation $1/w_{t,i}$, given $\hat{L}_t$ and $a_{t,i}$, which satisfies
    \begin{equation}\label{eq:square_expectation}
        \E\qty[\widehat{{w_{t,i}^{-1}}}^2\m\hat{L}_t,a_{t,i}]=\frac{2}{w_{t,i}^2}-\frac{1}{w_{t,i}}\leq \frac{2}{w_{t,i}^2}.
    \end{equation}
    Since $\hat{\ell}_{t,i}=\qty(\ell_{t,i}\widehat{{w_{t,i}^{-1}}})e_i$ when $a_{t,i}=1$, for $\dis\in\qty{\Fdis,\Pdis}$ we obtain
    \begin{align*}
        &\E\qty[\hat{\ell}_{t,i}\qty(\phi_i\qty(\eta \hat{L}_t;\dis)-\phi_i\qty(\eta \qty(\hat{L}_t+\hat{\ell}_t);\dis))\m \hat{L}_t]\\
        \leq&\E\qty[\ind\qty[a_{t,i}=1]\ell_{t,i}\widehat{{w_{t,i}^{-1}}}\cdot(\alpha+1)\eta \ell_{t,i}  J_i\qty(\eta \hat{L}_t;\dis)\widehat{w_{t,i}^{-1}}\m \hat{L}_t]\\
        \leq&2(\alpha+1)\eta\E\qty[w_{t,i}\frac{\ell^2_{t,i}J_i\qty(\eta \hat{L}_t;\dis)}{w^2_{t,i}}\m \hat{L}_t]\\
        \leq&2(\alpha+1)\eta\E\qty[\frac{J_i\qty(\eta \hat{L}_t;\dis)}{\phi_i\qty(\eta \hat{L}_t;\dis)}\m \hat{L}_t]\\
        \leq&
        \begin{cases}
            2(\alpha+1)\eta\qty(\frac{(\sigma_i\land m)+\frac{1}{\alpha}}{(\sigma_i-m+1)\lor 1})^{\frac{1}{\alpha}}, & \text{if } \dis=\Fdis,\\
            4\alpha\eta\qty(\frac{(\sigma_i\land m)+\frac{1}{\alpha}}{\sigma_i})^{\frac{1}{\alpha}}, & \text{if } \dis=\Pdis.
        \end{cases}
        \tag*{(by Lemma~\ref{lem:sigma_i})}
    \end{align*}
\end{proof}

\subsection{Proof of Lemma~\ref{lem:stability}}

By Lemma~\ref{lem:each_term}, we have the following result.

\subsubsection{Fr\'{e}chet Distribution}

When $\dis=\Fdis$, we have
\begin{align*}
    &\E\qty[\hat{\ell}_t \qty(\phi\qty(\eta \hat{L}_t;\Fdis)-\phi\qty(\eta \qty(\hat{L}_t+\hat{\ell}_t);\Fdis))\m \hat{L}_t]\\
    =&\sum_{i\in[d]}2(\alpha+1)\eta\qty(\frac{(\sigma_i\land m)+\frac{1}{\alpha}}{(\sigma_i-m+1)\lor 1})^{\frac{1}{\alpha}}\\
    \leq&2(\alpha+1)\eta\qty(m+\frac{1}{\alpha})^{\frac{1}{\alpha}}\qty(m+\sum_{i=1}^{d-m+1} \frac{1}{\sqrt[\alpha]{i}})\\
    \leq&2(\alpha+1)\eta\qty(m+\frac{1}{\alpha})^{\frac{1}{\alpha}}\qty(m+1+\int_1^{d-m+1} x^{-1/\alpha} \dd x)\\
    =&2(\alpha+1)\eta\qty(m+\frac{1}{\alpha})^{\frac{1}{\alpha}}\qty(m+\frac{\alpha(d-m+1)^{1-1/\alpha}-1}{\alpha-1})\\
    \leq& 2(\alpha+1)\eta\qty(m+\frac{1}{\alpha})^{\frac{1}{\alpha}}\qty(m+\frac{\alpha}{\alpha-1}(d-m+1)^{1-1/\alpha}).
\end{align*}

\subsubsection{Pareto Distribution}

When $\dis=\Pdis$, we have
\begin{align*}
    \E\qty[\hat{\ell}_t\qty(\phi\qty(\eta \hat{L}_t;\Pdis)-\phi\qty(\eta \qty(\hat{L}_t+\hat{\ell}_t);\Pdis))\m \hat{L}_t]&=\sum_{i\in[d]}4\alpha\eta\qty(\frac{(\sigma_i\land m)+\frac{1}{\alpha}}{\sigma_i})^{\frac{1}{\alpha}}\\
    &=\sum_{i\in[d]}4\alpha\eta\qty(\frac{m+\frac{1}{\alpha}}{\sigma_i})^{\frac{1}{\alpha}}\\
    &\leq4\alpha\eta\qty(m+\frac{1}{\alpha})^{\frac{1}{\alpha}}\qty(1+\int_1^d x^{-1/\alpha} \dd x)\\
    &=4\alpha\eta\qty(m+\frac{1}{\alpha})^{\frac{1}{\alpha}}\frac{\alpha d^{1-1/\alpha}-1}{\alpha-1}\\
    &\leq \frac{4\alpha^2}{\alpha-1}\eta\qty(m+\frac{1}{\alpha})^{\frac{1}{\alpha}}d^{1-1/\alpha}.
\end{align*}

\section{Technical Lemmas}

\begin{lemma}[\textbf{Gautschi's inequality}]\label{lem:Gautschi}
    For $x > 0$ and $s \in (0,1)$,
    $$x^{1 - s} < \frac{\Gamma(x + 1)}{\Gamma(x + s)} < (x + 1)^{1 - s}.$$
\end{lemma}

\begin{lemma}\cite[Eq.~(3.7)]{malik1966exact}
    \label{lem:pareto_order_statistics}
    Let $X_{k,n}$ be the $k$-th order statistics of i.i.d.~RVs from $\mathcal{P}_{\alpha}$ for $k\in[n]$, where $\alpha>1$. 
    Then, we have
    \begin{equation*}
        \E[X_{k,n}] =
        \frac{\Gamma\qty(n+1)\Gamma\qty(n-k-\frac{1}{\alpha}+1)}{\Gamma\qty(n-k+1)\Gamma\qty(n-\frac{1}{\alpha}+1)}.
    \end{equation*}
\end{lemma}

\begin{lemma}\label{lem:order_statistics}
Let $F(x)$ and $G(x)$ be CDFs of some random variables
such that $G(x)\ge F(x)$ for all $x\in\mathbb{R}$.
Let $(X_1,X_2,\dots,X_n)$ (resp.~$(Y_1,Y_2,\dots,Y_n)$) be RVs i.i.d.~from $F$ (resp.~$G$), and
$X_{k,n}$ (resp.~$Y_{k,n}$) be its $k$-th order statistics for any $k \in [n]$.
Then, $\E[Y_{k,n}] \le \E[X_{k,n}]$ holds.
\end{lemma}

\begin{proof}
Let $U\in[0,1]$ be uniform random variable over $[0,1]$
and let $X=F^{-1}(U)$ and $Y=G^{-1}(U)$, where
$F^{-1}$ and $G^{-1}$ are the left-continuous inverses of $F$ and $G$, respectively.
Then, 
$Y\le X$ holds almost surely and the marginal distributions satisfy $X \sim F$ and $Y\sim G$.
Therefore, if when take $(X_1,Y_1),\dots, (X_n,Y_n)$ as i.i.d.~copies of this $(X,Y)$, we see that
$Y_{k,n}\le X_{k,n}$ holds almost surely, which proves the lemma.
\end{proof}

\begin{lemma}\label{lem:penalty_pareto_frechet}
Let
$X_{k,n}$ (resp.~$Y_{k,n}$) be the $k$-th order statistics
of i.i.d.~RVs from $\mathcal{P}_{\alpha}$ and $\mathcal{F}_{\alpha}$ for $k\in[n]$.
Then, $\E[Y_{k,n}] \le \E[X_{k,n}]+1$ holds.
\end{lemma}

\begin{proof}
Letting $F(x)$ and $G(x)$ be the CDFs of $\mathcal{P}_{\alpha}$ and $\mathcal{F}_{\alpha}$,
we have
\begin{align*}
G(x)&=\ind[x\geq 0]e^{-1/x^{\alpha}}\\
&\geq \ind[x\geq 0]\left(1-\frac{1}{x^{\alpha}}\right)\\
&\geq \ind[x-1\geq 0]\left(1-\frac{1}{((x-1)+1)^{\alpha}}\right)\\
& = F(x-1),
\end{align*}
where $F(x-1)$ is the CDF of $X+1$ for $X\sim \mathcal{P}_{\alpha}$.
Then, it holds from Lemma~\ref{lem:order_statistics} that
$\E[Y_{k,n}] \le \E[X_{k,n}+1]= \E[X_{k,n}]+1$.
\end{proof}



\section{Issues in Proof of Claimed Extension to the Monotone Decreasing Case}\label{sec:issue}

In this section, we use the same notation as \citet{zhan2025follow} and reconstruct their missing arguments in Lemma~4.1 in \citet{zhan2025follow}.
The claim proposed by them is as follows.
\begin{claim}[Lemma 4.1 in \citet{zhan2025follow}]
    For any $\mathcal{I}\subseteq \qty{1,2,\cdots,d}$, $i\notin \mathcal{I}$, $\lambda\in\mathbb{R}^d$ such that $\lambda_i\geq 0$ and any $N\geq 3$, let
    \begin{equation*}
        J_{i,N,\mathcal{I}}(\lambda)\coloneqq\int_0^\infty \frac{1}{(x+\lambda_i)^{N}}\prod_{q\in\mathcal{I}}\qty(1-F(x+\lambda_q))\prod_{q\notin\mathcal{I}}F(x+\lambda_q)\dd x.
    \end{equation*} 
    Then, for all $k>0$, $\frac{J_{i,N+k,\mathcal{I}}(\lambda)}{J_{i,N,\mathcal{I}}(\lambda)}$ is increasing in $\lambda_q\geq 0$ for $q\notin\mathcal{I}$ and $q\neq i$, while decreasing in $\lambda_q\geq 0$ for $q\in\mathcal{I}$.
\end{claim}
In their paper, they consider the case that $F(x)$ is the cumulative distribution function of Fr\'{e}chet distribution with shape $2$, which is expressed as
\begin{equation*}
    F(x)=e^{-1/x^2},\quad x\geq 0,
\end{equation*}
and thus we also adopt the same setting in this section.
For this claim, they only gave the proof of the monotonic increase for the former case and it is written that the monotonic decrease for the latter case ``can be shown by the same argument''. Nevertheless, the monotonic decrease is not proved by the same argument, and it is highly likely that the statement itself does not hold as we will demonstrate below.

Now we follow the line of the proof of monotone increasing case in \citet{zhan2025follow}, giving an attempt to prove that
for all $k>0$, $\frac{J_{i,N+k,\mathcal{I}}(\lambda)}{J_{i,N,\mathcal{I}}(\lambda)}$ is monotonically decreasing on $\lambda_q\geq 0$ for $q\in\mathcal{I}$. 
For $q_0\in\mathcal{I}$, denote 
\begin{align*}
    J^{q_0}_{i,N,\mathcal{I}}(\lambda)&=
    -\frac{1}{2}\frac{\partial}{\partial \lambda_{q_0}}J_{i,N,\mathcal{I}}(\lambda)\\
    &=\int_0^\infty \frac{e^{-1/(x+\lambda_{q_0})^2}}{(x+\lambda_i)^{N}(x+\lambda_{q_0})^{3}}\prod_{q\in\mathcal{I}\setminus\qty{q_0}}\qty(1-F(x+\lambda_q))\prod_{q\notin\mathcal{I}}F(x+\lambda_q)\dd x.
\end{align*}
Hence,
\begin{equation}\label{eq:issue_derivative}
    \frac{\partial}{\partial \lambda_{q_0}}\frac{J_{i,N+k,\mathcal{I}}(\lambda)}{J_{i,N,\mathcal{I}}(\lambda)}=
    -2\cdot\frac{J^{q_0}_{i,N+k,\mathcal{I}}(\lambda)J_{i,N,\mathcal{I}}(\lambda)-J_{i,N+k,\mathcal{I}}(\lambda)J^{q_0}_{i,N,\mathcal{I}}(\lambda)}{J_{i,N,\mathcal{I}}(\lambda)^2}.
\end{equation}

Letting $Q(x)=\frac{1}{(x+\lambda_i)^{N}}\prod_{q\in\mathcal{I}\setminus\qty{q_0}}\qty(1-F(x+\lambda_q))\prod_{q\notin\mathcal{I}}F(x+\lambda_q)$, we have
\begin{align*}
    &J^{q_0}_{i,N+k,\mathcal{I}}(\lambda)J_{i,N,\mathcal{I}}(\lambda)\\
    =&\iint_{x,y\geq 0} \frac{e^{-1/(x+\lambda_{q_0})^2}}{(x+\lambda_i)^{k}(x+\lambda_{q_0})^{3}}\qty(1-e^{-1/(y+\lambda_{q_0})^2})Q(x)Q(y)\dd x\dd y\\
    =&\frac{1}{2}\iint_{x,y\geq 0} Q(x)Q(y)\qty[\frac{e^{-1/(x+\lambda_{q_0})^2}}{(x+\lambda_i)^{k}(x+\lambda_{q_0})^{3}}\qty(1-e^{-1/(y+\lambda_{q_0})^2})+
    \frac{e^{-1/(y+\lambda_{q_0})^2}}{(y+\lambda_i)^{k}(y+\lambda_{q_0})^{3}}\qty(1-e^{-1/(x+\lambda_{q_0})^2})
    ]\dd x\dd y.
\end{align*}
Here, it is worth noting that $\frac{e^{-1/(x+\lambda_{q_0})^2}}{(x+\lambda_i)^{k}(x+\lambda_{q_0})^{3}}$ and $\qty(1-e^{-1/(y+\lambda_{q_0})^2})$ share no common factors, and thus no factor can be hidden in $Q(x)$.

Similarly, we have
\begin{align*}
    &J_{i,N+k,\mathcal{I}}(\lambda)J^{q_0}_{i,N,\mathcal{I}}(\lambda)\\
    =&\iint_{x,y\geq 0} \frac{1-e^{-1/(x+\lambda_{q_0})^2}}{(x+\lambda_i)^{k}}\frac{e^{-1/(y+\lambda_{q_0})^2}}{(y+\lambda_{q_0})^{3}}Q(x)Q(y)\dd x\dd y\\
    =&\frac{1}{2}\iint_{x,y\geq 0} Q(x)Q(y)\qty[\frac{1-e^{-1/(x+\lambda_{q_0})^2}}{(x+\lambda_i)^{k}}\frac{e^{-1/(y+\lambda_{q_0})^2}}{(y+\lambda_{q_0})^{3}}+
    \frac{1-e^{-1/(y+\lambda_{q_0})^2}}{(y+\lambda_i)^{k}}\frac{e^{-1/(x+\lambda_{q_0})^2}}{(x+\lambda_{q_0})^{3}}
    ]\dd x\dd y.
\end{align*}
Now we substitute the above two expressions into \eqref{eq:issue_derivative}. 

By an elementary calculation, we have
\begin{align*}
    &\frac{e^{-1/(x+\lambda_{q_0})^2}}{(x+\lambda_i)^{k}(x+\lambda_{q_0})^{3}}\qty(1-e^{-1/(y+\lambda_{q_0})^2})+
    \frac{e^{-1/(y+\lambda_{q_0})^2}}{(y+\lambda_i)^{k}(y+\lambda_{q_0})^{3}}\qty(1-e^{-1/(x+\lambda_{q_0})^2})\\
    -&\frac{1-e^{-1/(x+\lambda_{q_0})^2}}{(x+\lambda_i)^{k}}\frac{e^{-1/(y+\lambda_{q_0})^2}}{(y+\lambda_{q_0})^{3}}-
    \frac{1-e^{-1/(y+\lambda_{q_0})^2}}{(y+\lambda_i)^{k}}\frac{e^{-1/(x+\lambda_{q_0})^2}}{(x+\lambda_{q_0})^{3}}\\
    =&\frac{(y+\lambda_i)^{k}-(x+\lambda_i)^{k}}{(x+\lambda_i)^{k}(y+\lambda_i)^{k}}\qty(\frac{e^{-1/(x+\lambda_{q_0})^2}\qty(1-e^{-1/(y+\lambda_{q_0})^2})}{(x+\lambda_{q_0})^{3}}-\frac{e^{-1/(y+\lambda_{q_0})^2}\qty(1-e^{-1/(x+\lambda_{q_0})^2})}{(y+\lambda_{q_0})^{3}})\\
    =&\frac{\qty((y+\lambda_i)^{k}-(x+\lambda_i)^{k})\qty(1-e^{-1/(y+\lambda_{q_0})^2})\qty(1-e^{-1/(x+\lambda_{q_0})^2})}{(x+\lambda_i)^{k}(y+\lambda_i)^{k}}\\
    \cdot&\qty(\frac{e^{-1/(x+\lambda_{q_0})^2}}{(x+\lambda_{q_0})^{3}\qty(1-e^{-1/(x+\lambda_{q_0})^2})}-\frac{e^{-1/(y+\lambda_{q_0})^2}}{(y+\lambda_{q_0})^{3}\qty(1-e^{-1/(y+\lambda_{q_0})^2})}).\numberthis\label{eq:issue_h}
\end{align*}
Here, for \eqref{eq:issue_h}, one can see that when $x\geq y$, the first term becomes negative. 
On the other hand, when $x\leq y$, the first term becomes positive. 
Define
\begin{equation*}
    h(x)=\frac{e^{-1/x^2}}{x^{3}\qty(1-e^{-1/x^2})}.
\end{equation*}
If $h(x)$ can be shown to be monotonically decreasing,  then \eqref{eq:issue_derivative} is negative, that is, the claim holds. 
However, the derivative of $h(x)$ is given as
\begin{equation*}
    h'(x)=\frac{e^{-1/x^2}\qty(3x^2\qty(e^{-1/x^2}-1)+2)}
       {x^6\qty(1 - e^{-1/x^2})^2},
\end{equation*}
where one can see that $h'(x)$ is not always negative, and thus $h(x)$ is not monotonic in $x\in[0,\infty)$.
\begin{figure}[t]
    \centering
    \includegraphics[width=0.65\textwidth]{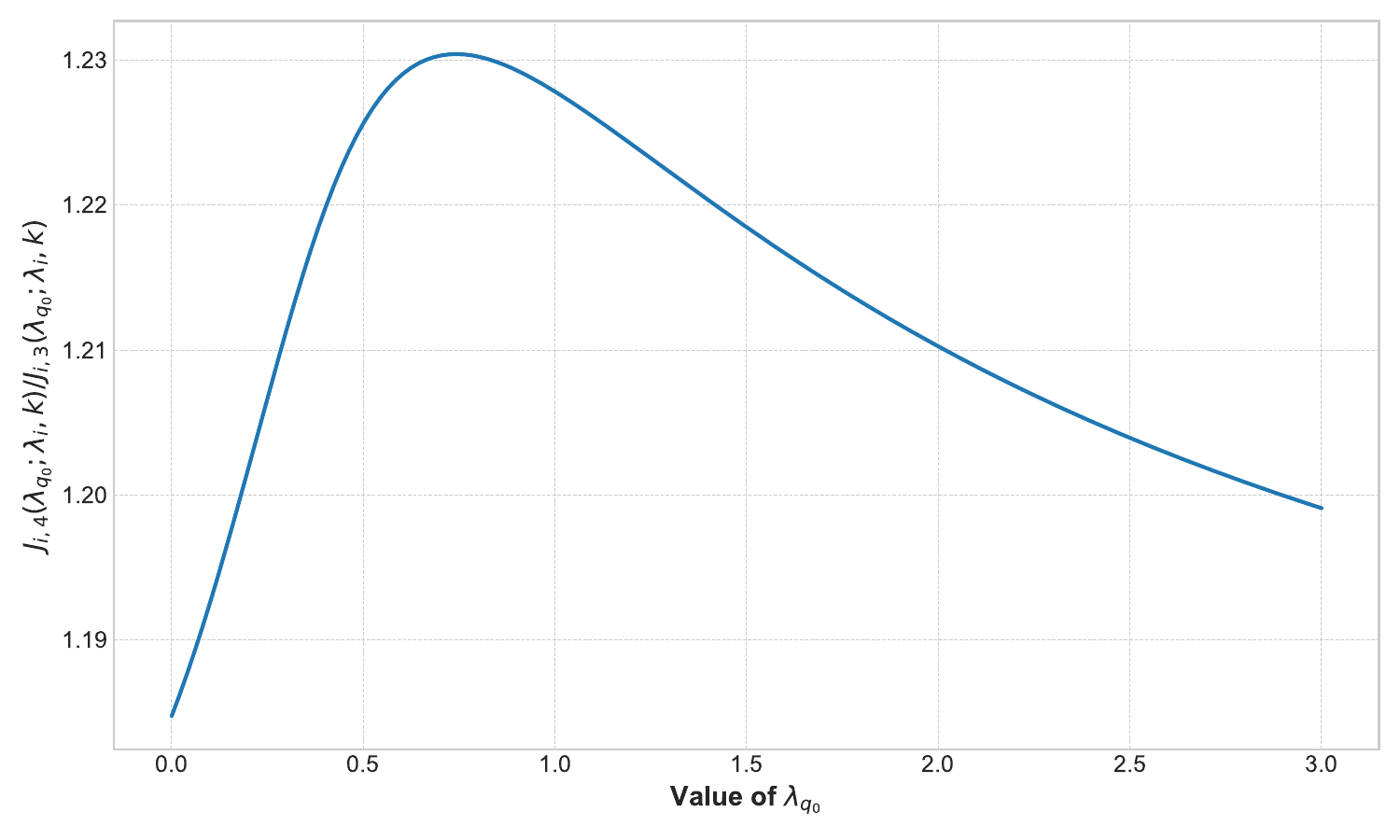}
    \caption{The ratio $J_{i,4,\mathcal{I}}(\lambda)/J_{i,3,\mathcal{I}}(\lambda)$ as a function of $\lambda_{q_0}$ for $q_0\in\mathcal{I}\setminus\qty{i}$.}
    \label{fig:issue_function}
\end{figure}
By the analysis above, we can see that the monotone decreasing case can not be proved by the same argument as in increasing case.

\paragraph{Numerical Simulation} 
We now turn to a numerical simulation implemented in Python to illustrate this failure. 
As a counterexample, it is sufficient to show one special case where $\left\lvert\mathcal{I}\right\rvert=5$, $q_0\in\mathcal{I}$ and $\lambda_q=\lambda_i=0.5$ for all $q\neq q_0$.
Then, we consider the expression
\begin{equation*}
    \frac{J_{i,4,\mathcal{I}}(\lambda_{q_0},\lambda_i)}{J_{i,3,\mathcal{I}}(\lambda_{q_0},\lambda_i)}=\frac{\int_0^\infty \frac{1}{(x+\lambda_i)^{4}}\qty(1-F(x+\lambda_{q_0}))\qty(1-F(x+\lambda_i))^4 F(x+\lambda_i)^{2}\dd x}{\int_0^\infty \frac{1}{(x+\lambda_i)^{3}}\qty(1-F(x+\lambda_{q_0}))\qty(1-F(x+\lambda_i))^4 F(x+\lambda_i)^{2}\dd x}.
\end{equation*}
Here, we set $N=3$ and $k=1$, which are the parameters used in the subsequent analysis in  \citet{zhan2025follow}. 
Treating this expression as a function of $\lambda_{q_0}$, we can observe from Figure~\ref{fig:issue_function} that the function is not monotonic in $\lambda_q$, 
and it does not hold that $\frac{J_{i,4,\mathcal{I}}(\lambda_{q_0},\lambda_i)}{J_{i,3,\mathcal{I}}(\lambda_{q_0},\lambda_i)}\leq\frac{J_{i,4,\mathcal{I}}(\lambda_i,\lambda_i)}{J_{i,3,\mathcal{I}}(\lambda_i,\lambda_i)}$ for some $\lambda_{q_0}>\lambda_i$.

From these arguments, it is highly likely that Lemma~4.1 in \citet{zhan2025follow} does not hold, and at least the current proof is not complete. On the other hand, our analysis is constructed in a way that does not need the monotonicity of $\frac{J_{i,N+k,\mathcal{I}}(\lambda)}{J_{i,N,\mathcal{I}}(\lambda)}$, which becomes the key to the extension of the results for the MAB to the combinatorial semi-bandit setting.

\section*{Acknowledgements}
We thank Dr. Jongyeong Lee for pointing out the error in the proof of Lemma~\ref{lem:penalty_bound} in the previous version.

\bibliographystyle{plainnat}
\bibliography{reference}

\end{document}